\documentclass[twoside]{article}

\usepackage{microtype}
\usepackage{graphicx}
\usepackage{subcaption}
\usepackage{url}
\usepackage{amsthm}
\usepackage{amssymb}
\newtheorem{theorem}{Theorem}
\usepackage{booktabs} 
\usepackage[ruled,vlined]{algorithm2e}
\usepackage[preprint]{aistats2027}
\usepackage{cleveref}
\usepackage[round]{natbib}

\begin{document}

%
\runningtitle{ISBO: Scalable Spatio-Temporal Bayesian Optimization}

%

\twocolumn[

\aistatstitle{ISBO: Scalable Spatio-Temporal Bayesian Optimization with Log Gaussian Cox Process Models via the INLA-SPDE Approach}

\aistatsauthor{ Kaichuang Yang \And Håvard Rue  \And  Jakob Zeitler}

\aistatsaddress{KAUST\\University of Oxford \And KAUST \And University of Oxford } ]

\begin{abstract}
Bayesian Optimization (BO) is a popular method for efficiently optimizing expensive black-box objectives. However, BO utilizing standard Gaussian Processes is ill-suited for doubly stochastic Cox Processes that are often used in spatio-temporal problem spaces. We introduce INLA–SPDE Spatio-Temporal Bayesian Optimization (ISBO): the first scalable BO framework for spatio-temporal data, that models the log-intensity with a Log-Gaussian Cox Process(LGCP) and performs inference via Integrated Nested Laplace Approximation and Stochastic Partial Differential Equations (INLA–SPDE) approach. Using a Matérn field on meshes yields a sparse Gaussian Markov Random Field, where INLA provides fast and accurate posterior inference throughout sequential optimization. ISBO stably locates high-intensity regions and the peak of the latent intensity with minimal evaluations. A time-varying Upper Confidence Bound acquisition with masking avoids revisits, while penalized-complexity priors regularize early rounds. Experiments on synthetic and real-world spatio-temporal datasets show accurate peak discovery, intensity recovery, and substantial speedups over an RKHS-based baseline, positioning ISBO as a practical choice for BO with point-process data.
\end{abstract}

\section{Introduction}

Optimization problems arise widely in science and engineering, where the search for good settings often requires costly evaluations of black-box objectives. Bayesian Optimization (BO) is a sample-efficient optimization method that uses probabilistic surrogate models trained on samples of the unknown objective function and selects new points to sample by maximizing an acquisition function over the surrogate. BO has shown strong results in materials design \citep{materials_design}, experimental physics \citep{exp_physics} and reinforcement learning \citep{reinforcement_learning}. 

BO generally assumes each query to return a real-valued, possibly noisy, function value and models it with Gaussian Process (GP) regression. This yields a closed-form posterior on which we can apply acquisition functions such as Upper Confidence Bounding (UCB) \citep{lai1985UCB}, Probability of Improvement (PI) \citep{kushner1964PI}, and Expected Improvement (EI) \citep{movckus1974EI}. However, this standard setup does not fit spatio-temporal problems where the data come as point events generated by a latent intensity. Here, we define intensity as the instantaneous event rate i.e. expected events per unit time for temporal data and per unit area for spatial data (also referred to as the intensity surface). Over small intervals or regions, expected counts scale with length or area, and higher intensity means more frequent events. The observed pattern is doubly stochastic for such a process: conditional on the latent intensity, the points follow an inhomogeneous Poisson process, while the intensity itself is random as a realization of a stochastic field. Thus, randomness enters twice through Poisson sampling and the stochastic intensity. Typical examples include locations of forest trees, the GPS coordinates of disease cases, arrivals in transport, or events in finance and neuroscience. In these cases, a natural model choice is the (spatio-)temporal Cox Process (CP), specifically the Gaussian Cox Process (GCP), where the intensity is a Gaussian random field through a link function. 

BO over GCPs requires reliable posterior uncertainty in the latent field, since acquisition functions depend on both posterior mean and uncertainty. Bayesian GCP inference methods such as MCMC and variational inference can provide this uncertainty, but may be computationally expensive for sequential optimization, where the model must be repeatedly refitted after each query. See \citet{diggle2013sp_point_pattern} and \cite{flagg2023LGCPINLAReview} for an overview of spatial and spatio-temporal point processes.

In this paper we propose INLA–SPDE Spatio-Temporal Bayesian Optimization (ISBO): a scalable BO framework for temporal, spatial, and spatio-temporal point-process data that combines the log-Gaussian Cox Process (LGCP) with the Integrated Nested Laplace Approximation (INLA) and the Stochastic Partial Differential Equation (SPDE) approach. As far as we are aware, the first work of this kind is \citet{meiBayesianOptGCPMstD}, which we adopt as our baseline and improve over. We additionally compare against the fast LGCP inference method of \citet{dovers2023fast} at the inference level. Specifically, ISBO is the first computationally scalable solution for Bayesian Optimization over spatio-temporal data, enabled by the INLA-SPDE approach.

\paragraph{Our contributions}
\begin{itemize}

    \item We develop ISBO, a scalable Bayesian optimization framework for point-process data, where each query reveals events within a local region rather than a noisy scalar evaluation, and sequentially updates a probabilistic surrogate to optimize the latent intensity over temporal and spatial domains.

    \item We introduce a latent-scale, region-aware optimization strategy combining time-varying UCB, masking, and region-based regret measures tailored to point-process observations.

    \item We demonstrate substantial improvements of ISBO over existing baselines on both synthetic and real-world datasets in terms of optimization performance and computational scalability, and conduct detailed ablation studies to quantify the contribution of its key model components.
    
\end{itemize}
\section{Related Work}

\paragraph{Log-Gaussian Cox Process Models}

We focus on LGCPs, i.e., Cox Processes where the logarithm of the intensity is modeled as a Gaussian process \citep{moller1998LGCP}. Because a Cox Process is doubly stochastic, intensity estimation is more challenging than for deterministic intensity \citep{flaxman2017poisson_inten}. LGCPs serve as the natural analogue of linear Gaussian models for point-process data in geostatistics and provide a flexible, tractable class for spatial prediction \citep{Diggle_LGCP}. 

\paragraph{Integrated Nested Laplace Approximation}

Integrated Nested Laplace Approximation (INLA) is better viewed as an evolving framework for approximate Bayesian inference in latent Gaussian models than as a single fixed algorithm. Its development spans Gaussian Markov random fields \citep{rue2005GMRF}, the original INLA formulation \citep{rue2009INLA}, the SPDE representation of Mat\'ern fields \citep{lindgren2011explicit_GMRF_SPDE}, off-grid LGCP inference that retains exact event locations \citep{simpson2016going}, penalized-complexity priors \citep{PCprior}, modern latent-model reformulations \citep{van2023new}, and low-rank variational Bayes corrections \citep{van2024low}. The framework continues to evolve, including recent extensions toward high-performance computing \citep{11017730} and Python implementations \citep{fattah2026pyinlafastbayesianinference}. For LGCP inference, INLA provides fast deterministic posterior approximations and has become widely used in spatial statistics \citep{rue2017INLA_review,bakka2018spatial_review}. Compared with MCMC, INLA can achieve substantial computational savings while retaining accurate posterior uncertainty estimates \citep{taylor2014inlaMCMC}. This speed--accuracy trade-off is particularly attractive for BO, where the LGCP surrogate must be repeatedly refitted after each query.

\paragraph{Baseline 1: Laplace Approximation with Kernel Transformation in Reproducing Kernel Hilbert Space (RKHS)}

\citet{meiBayesianOptGCPMstD}  propose a MAP-based inference framework for Gaussian Cox processes. A Laplace approximation is first employed to obtain a Gaussian approximation of the posterior over the latent field. The resulting MAP objective is then reformulated in an RKHS by introducing a kernel transformation that absorbs the integral term into the RKHS norm. This leads to a well-posed optimization problem admitting a finite-dimensional representation via the representer theorem. For practical computation, a Nyström approximation is further developed to efficiently approximate the transformed kernel and enable posterior mean and covariance estimation for the latent intensity. Parts of the derivation assume a large-$n$ regime where the likelihood dominates the prior and Laplace normalization terms (equation (5) in \citep{meiBayesianOptGCPMstD}), an assumption that may not hold in early BO rounds. In our reproductions and additional tests, we observed notable run-to-run variability in the optimization trajectory, sensitivity to initialization and hyperparameters, and limited scalability with respect to sample size and domain resolution. These factors lead to unstable acquisitions and slower convergence, with higher computational cost compared to a SPDE-INLA approach. This is both BO and inference level baseline.

\paragraph{Baseline 2: Fast Approximate LGCP Inference}

\citet{dovers2023fast} propose a fast approximate maximum-likelihood framework for LGCPs based on reduced-rank interpolation, likelihood approximation, and automatic differentiation. Their method combines either variational or Laplace approximation with a low-rank representation of the latent Gaussian field, substantially reducing the computational cost of evaluating the high-dimensional marginal likelihood. Automatic differentiation is used to provide exact gradients for efficient optimization. Their experiments report substantial speed-ups over the INLA implementation available at the time, mainly due to rank reduction and automatic differentiation, with only a small loss in model fit \citep{dovers2023fast}.  In our experiments, it is faster than the inference procedure of \citet{meiBayesianOptGCPMstD}, but slower than modern INLA, its runtime increases with the number of events, while modern INLA doesn't. This is inference level only baseline.

\paragraph{Bayesian Optimization through Acquisition Functions} 

BO is a methodology for optimizing expensive objective functions that has proven success in the sciences, engineering, and beyond \citep{garnett_bayesoptbook_2023}. Most BO methods adopt GPs as the surrogate model. More recently, neural-network surrogates have been explored \citep{dai2022sampleNNSurog,swersky2020BONNSurog,li2023studyNNSurog}, alongside tree-based surrogates \citep{kim2022treeSurog,liang2021benchmarkingtreeSurog,shen2023proxybotreeSurog}. Much of the BO literature focuses on continuous domains, whereas many real applications involve discrete temporal or spatial point events, which are the focus of our work. Apart from \citet{meiBayesianOptGCPMstD}, which serves as our baseline, existing BO research does not address GCP models, where the Poisson event rate is modeled by a latent Gaussian process. In the remainder of this paper, we make our original contribution (ISBO) to this emerging area of research: an INLA-SPDE based approach enabling scalable BO for temporal, spatial, and spatio-temporal point-process data.

\section{Methodology}
Our method has two main components: a scalable LGCP surrogate and BO over the latent intensity. The surrogate combines off-grid SPDE modeling \citep{simpson2016going}, modern INLA inference \citep{van2023new}, and PC priors \citep{PCprior} for efficient uncertainty quantification during sequential optimization. The BO component uses a time-varying latent-scale UCB with region-based querying and masking to avoid redundant observations. Additional background and derivations are provided in Appendix~\ref{app:LGMs}.

\subsection{LGCP Surrogate Model with Modern INLA--SPDE}
\label{sec:lgcp-inla-spde}

We model the latent event intensity with a Log-Gaussian Cox Process (LGCP) and use its posterior as the surrogate model for Bayesian optimization. Let $\mathcal D\subset\mathbb R^d$ denote the observation domain and $\mathcal A=\{s_i\}_{i=1}^{n}$ the currently observed event locations.
An LGCP specifies a stochastic intensity
\begin{align}
    \lambda(s)
    &=
    \exp\{\eta(s)\},
    \qquad s\in\mathcal D,
    \label{eq:lgcp-link}
\end{align}
where $\eta(s)$ is a latent Gaussian linear predictor. Conditional on $\lambda(\cdot)$, the point pattern is an inhomogeneous Poisson process. Hence, for any measurable region $\mathcal B\subseteq\mathcal D$,
\begin{align}
    N(\mathcal B)\mid\lambda
    &\sim
    \operatorname{Poisson}
    \left(
        \int_{\mathcal B}\lambda(s)\,ds
    \right),
\end{align}
where $N(\mathcal B)$ denotes the number of events in $\mathcal B$.
For the observed pattern $\mathcal A$, the log-likelihood is
\begin{align}
    \log p(\mathcal A\mid\eta)
    &=
    \sum_{i=1}^{n}\eta(s_i)
    -
    \int_{\mathcal D}\exp\{\eta(s)\}\,ds
    + C,
    \label{eq:lgcp-loglikelihood}
\end{align}
where $C$ does not depend on the latent field.

We model the linear predictor as
\begin{align}
    \eta(s)
    &=
    \alpha_0
    +
    U(s)
    +
    \mathbf z(s)^\top\boldsymbol{\beta},
    \label{eq:lgcp-linear-predictor}
\end{align}
where $\alpha_0$ is an intercept, $U(s)$ is a zero-mean Mat\'ern Gaussian field, $\mathbf z(s)$ is a vector of optional explanatory variables, and $\boldsymbol{\beta}$ contains the corresponding regression coefficients.

\paragraph{Off-grid SPDE representation.}
A common approach to LGCP inference discretizes $\mathcal D$ into a fine regular lattice and approximates the number of events in each cell by a Poisson random variable. Such an approximation depends on the chosen grid resolution and effectively bins events into cells. Following the off-grid approach of \citet{simpson2016going}, we instead retain the exact event locations and represent the Gaussian field continuously through a finite-dimensional basis expansion.

Using the SPDE approach of \citet{lindgren2011explicit_GMRF_SPDE},
\begin{align}
    U(s)
    &\approx
    \sum_{r=1}^{m}\psi_r(s)w_r,
    \label{eq:spde-fem}
\end{align}
where $\{\psi_r(s)\}_{r=1}^{m}$ are piecewise-linear finite-element basis functions defined on a triangulated mesh and $\mathbf w=(w_1,\ldots,w_m)^\top$ is the corresponding vector of Gaussian weights. The SPDE construction gives $\mathbf w$ a Gaussian Markov random field representation with a sparse precision matrix determined by the Mat\'ern covariance parameters. The triangulated mesh can be refined locally where greater spatial resolution is required, without imposing a uniformly fine lattice over the whole domain.

\paragraph{Point-process integration.}
Although observed events retain their exact locations, the integral in
\eqref{eq:lgcp-loglikelihood} must be approximated numerically. Let
\begin{align}
    \mathrm{ips}^{s}
    &=
    \{(\xi_k,\omega_k)\}_{k=1}^{q}
    \label{eq:ips-space-default}
\end{align}
denote $q$ integration points $\xi_k$ and their non-negative quadrature weights $\omega_k$. Then
\begin{align}
    \int_{\mathcal D}\lambda(s)\,ds
    &\approx
    \sum_{k=1}^{q}
    \omega_k\lambda(\xi_k).
    \label{eq:lgcp-quadrature}
\end{align}
This gives the augmented Poisson representation used for LGCP inference: the integration points approximate the integral term, while the observed events contribute at their original locations.

\paragraph{Modern INLA formulation.}
Following \citet{van2023new}, the linear predictors are not included as explicit components of the latent field. For our LGCP, we define
\begin{align}
    \mathbf{x}
    &=
    \left(
        \alpha_0,\,
        \mathbf w^\top,\,
        \boldsymbol{\beta}^\top
    \right)^\top,
    \quad
    \mathbf{x}\mid\boldsymbol{\theta}
    \sim
    \mathcal N
    \left(
        \mathbf 0,
        Q(\boldsymbol{\theta})^{-1}
    \right),
    \label{eq:modern-latent-field}
\end{align}
where $\mathbf w$ represents the Mat\'ern field through \eqref{eq:spde-fem}, $\boldsymbol{\theta}$ denotes the non-Gaussian hyperparameters, and $Q(\boldsymbol{\theta})$ is the sparse latent precision matrix. In particular, $\boldsymbol{\theta}$ includes the Mat\'ern--SPDE range and marginal standard deviation parameters.

For any collection of observation, integration, or prediction locations, the corresponding linear predictors are obtained through
\begin{align}
    \boldsymbol{\eta}
    &=
    A\mathbf{x},
    \label{eq:modern-eta}
\end{align}
where $A$ is a sparse projection matrix. Its rows evaluate the intercept, covariates, and finite-element basis functions at the corresponding locations.

For fixed $\boldsymbol{\theta}$, modern INLA constructs the Gaussian
approximation
\begin{align}
    \mathbf{x}\mid\boldsymbol{\theta},\mathbf y
    &\approx
    \mathcal N
    \left(
        \boldsymbol{\mu}(\boldsymbol{\theta}),
        Q_X(\boldsymbol{\theta})^{-1}
    \right),
    \label{eq:modern-lgcp-gaussian}
\end{align}
where $\mathbf y$ denotes the augmented likelihood representation and
\begin{align}
    Q_X(\boldsymbol{\theta})
    &=
    Q(\boldsymbol{\theta})
    +
    A^\top D(\boldsymbol{\theta})A.
    \label{eq:modern-lgcp-posterior-precision}
\end{align}
Here $D(\boldsymbol{\theta})$ is a diagonal matrix containing the local curvature contributions from the likelihood. Since both $Q(\boldsymbol{\theta})$ and $A$ are sparse, the conditional posterior precision remains sparse.

\paragraph{Penalized complexity (PC) priors.}
We assign PC priors \citep{PCprior} to the key Mat\'ern--SPDE hyperparameters. For a latent effect with marginal standard deviation $\sigma>0$, the prior can be specified through
\begin{align}
    \Pr(\sigma>u)
    &=
    \alpha_{\mathrm{pc}},
    \qquad
    u>0,\quad
    0<\alpha_{\mathrm{pc}}<1,
    \label{eq:pc-prior-sigma}
\end{align}
where $(u,\alpha_{\mathrm{pc}})$ controls the scale and upper-tail probability. We analogously specify a PC prior for the practical Mat\'ern range. These priors regularize the latent field when relatively few events have been revealed, which is particularly relevant during early BO iterations. Further details are given in Appendix~\ref{app:PC_prior}.

\paragraph{Posterior surrogate.}
For a prediction location $s$, let $\mathbf a(s)^\top$ denote the corresponding row of $A$. From \eqref{eq:modern-lgcp-gaussian},
\begin{align}
    \eta(s)\mid\boldsymbol{\theta},\mathbf y
    &\approx
    \mathcal N
    \left(
        \mu_{\eta}(s;\boldsymbol{\theta}),
        \sigma_{\eta}^{2}(s;\boldsymbol{\theta})
    \right),
    \label{eq:eta-conditional-posterior}
\end{align}
where
\begin{align}
    \mu_{\eta}(s;\boldsymbol{\theta})
    &=
    \mathbf a(s)^\top
    \boldsymbol{\mu}(\boldsymbol{\theta}),
    \\
    \sigma_{\eta}^{2}(s;\boldsymbol{\theta})
    &=
    \mathbf a(s)^\top
    Q_X(\boldsymbol{\theta})^{-1}
    \mathbf a(s).
    \label{eq:eta-conditional-moments}
\end{align}
Marginalization over the INLA support points
$\{\boldsymbol{\theta}_k\}_{k=1}^{K}$ gives
\begin{align}
    \tilde p(\eta(s)\mid\mathbf y)
    &\approx
    \sum_{k=1}^{K}
    p_G
    \left(
        \eta(s)
        \mid
        \boldsymbol{\theta}_k,\mathbf y
    \right)
    \tilde p
    \left(
        \boldsymbol{\theta}_k\mid\mathbf y
    \right)
    \delta_k ,
    \label{eq:eta-marginal-posterior}
\end{align}
where $\delta_k$ denotes the numerical integration weight. For the non-Gaussian LGCP likelihood, we additionally use the low-rank variational Bayes correction described in Appendix~\ref{app:LRCorVB} to improve the posterior mean.

At BO iteration $t$, let $\mathcal A_t$ denote the events revealed up to iteration $t$. Refitting the LGCP yields
\begin{align}
    \mu_{\eta,t}(x)
    &=
    \mathbb E[\eta(x)\mid\mathcal A_t],
    \sigma_{\eta,t}^{2}(x)
    =
    \operatorname{Var}[\eta(x)\mid\mathcal A_t],
    \label{eq:isbo-surrogate}
\end{align}
which form the probabilistic surrogate used by the acquisition rule in \cref{sec:bo-over-intensity}. Since $\lambda(x)=\exp\{\eta(x)\}$, posterior summaries of the intensity are obtained directly by transforming the posterior of $\eta(x)$.

The SPDE mesh, integration points, and BO candidate grid have distinct roles: the mesh represents the latent Mat\'ern field, the integration points approximate the point-process likelihood, and the BO grid determines where the acquisition function is evaluated. They therefore need not coincide.

For temporal point processes, the same formulation applies on $\mathcal T=[t_{\min},t_{\max}]$ with $\lambda(t)=\exp\{\eta(t)\}$ and a one-dimensional SPDE mesh.

\subsection{Bayesian Optimization over Estimated Intensity}
\label{sec:bo-over-intensity}

We use the LGCP posterior from
\cref{sec:lgcp-inla-spde} as the surrogate model for Bayesian optimization. Unlike standard BO, a query in the point-process setting does not directly observe a noisy value of the objective at a selected point. Instead, selecting a query location reveals the point-process events within a local region around that location, after which the LGCP surrogate is refitted.

\paragraph{Optimization objective and region-based queries.}
Let $\mathcal G$ denote a finite set of candidate query locations and define
\begin{align}
    x^\star
    &\in
    \arg\max_{x\in\mathcal G}\lambda(x)
    \label{eq:bo-optimum}
\end{align}
as a maximizer of the latent intensity. Our objective is to identify $x^\star$, or an associated high-intensity region, using as few revealed regions as possible.

At BO iteration $t$, selecting a centre $x_t\in\mathcal G$ reveals all point-process events within
\begin{align}
    \mathcal W(x_t;r)
    &=
    \{x:\lVert x-x_t\rVert\leq r\}\cap\mathcal D,
    \label{eq:bo-query-window}
\end{align}
where $r>0$ is the query radius. In the one-dimensional temporal case,
\begin{align}
    \mathcal W(x_t;r)
    &=
    [x_t-r,x_t+r]\cap\mathcal T .
\end{align}
The newly revealed events are added to the current event set $\mathcal A_t$, and the LGCP is subsequently refitted. Hence, the sequential update can be summarized as
\begin{align}
    \mathcal A_t
    \longrightarrow
    p(\eta\mid\mathcal A_t)
    \longrightarrow
    x_{t+1}
    \longrightarrow
    \mathcal W(x_{t+1};r)
    \longrightarrow
    \mathcal A_{t+1}.
    \label{eq:isbo-update}
\end{align}

\begin{theorem}[Peak equivalence under the LGCP transformation]
\label{thm:latent-scale-peak}
Let $\lambda(x) =\exp\{\eta(x)\}.$ Then the latent log-intensity and the intensity have the same set of
maximizers,
\begin{align}
    \arg\max_{x\in\mathcal G}\eta(x)
    =
    \arg\max_{x\in\mathcal G}\lambda(x).
    \label{eq:eta-lambda-argmax}
\end{align}
Moreover, if
\begin{align}
    \eta(x)\mid\mathcal A_t
    \approx
    \mathcal N\!\left(
        \mu_{\eta,t}(x),
        \sigma_{\eta,t}^{2}(x)
    \right),
\end{align}
then we have $\operatorname{sd}
    \left[
        \lambda(x)\mid\mathcal A_t
    \right]$ equals to
\begin{align}
    \exp\!\left\{
        \mu_{\eta,t}(x)
        +
        \frac{1}{2}\sigma_{\eta,t}^{2}(x)
    \right\}
    \sqrt{
        \exp\!\left\{
            \sigma_{\eta,t}^{2}(x)
        \right\}-1
    }.
    \label{eq:lambda-posterior-sd}
\end{align}
\end{theorem}

The proof is provided in Appendix~\ref{app:latent-scale-proof}. Theorem~\ref{thm:latent-scale-peak} motivates performing acquisition on the latent Gaussian scale. 

\paragraph{Latent-scale acquisition.}

The exponential transformation preserves the optimization target but induces a nonlinear amplification of posterior uncertainty on the intensity scale, particularly in data-sparse regions. Inspired by \citet{srinivas2009gaussian}, we therefore define the UCB-style acquisition directly using the posterior mean and standard deviation of the latent predictor,
\begin{align}
    a_t(x)
    &=
    \mu_{\eta,t}(x)
    +
    \kappa_t\sigma_{\eta,t}(x).
    \label{eq:isbo-acquisition}
\end{align}

For a BO budget of $T$ iterations, we use a decreasing exploration schedule, with $\kappa_{\mathrm{start}}
    \geq
    \kappa_{\mathrm{final}}>0:$
\begin{align}
    \kappa_t
    &=
    \kappa_{\mathrm{start}}
    -
    \frac{t-1}{T-1}
    \left(
        \kappa_{\mathrm{start}}
        -
        \kappa_{\mathrm{final}}
    \right).
    \label{eq:kappa-schedule}
\end{align}
Thus, earlier iterations place greater emphasis on posterior uncertainty, while later iterations increasingly favor exploitation of regions with high posterior log-intensity. The values of $\kappa_{\mathrm{start}}$ and $\kappa_{\mathrm{final}}$ are specified in the corresponding experimental settings.

\paragraph{Masking and sequential exploration.}
To avoid repeatedly querying already revealed regions, let
\begin{align}
    \mathcal R_t
    &=
    \bigcup_{j=1}^{t}
    \mathcal W(x_j;r)
    \label{eq:revealed-region}
\end{align}
denote the portion of the domain revealed by the first $t$ queries. The available candidate set is
\begin{align}
    \mathcal C_t
    &=
    \{x\in\mathcal G:x\notin\mathcal R_t\}.
    \label{eq:bo-candidate-set}
\end{align}
The next centre is selected by
\begin{align}
    x_{t+1}
    &=
    \arg\max_{x\in\mathcal C_t}
    a_t(x).
    \label{eq:bo-next-query}
\end{align}
After revealing $\mathcal W(x_{t+1};r)$, the newly observed events and the corresponding integration region are incorporated into the LGCP fit, and the posterior surrogate is updated. This mask prevents repeated acquisition from regions whose point-process events have already been revealed.

\begin{theorem}[Constant BO refitting time under modern INLA]
\label{thm:modern-inla-scaling}
Under the modern INLA formulation in \eqref{eq:modern-latent-field}, ISBO has constant refitting time across BO iterations as the number of revealed events increases, given the SPDE mesh and model specification remain fixed.
\end{theorem}

The proof is provided in Appendix~\ref{app:modern-inla-scaling-proof}.

For spatio-temporal data, the same temporal selection mechanism is used to reveal successive time windows. The spatial LGCP is then refitted using the spatial event locations whose event times fall inside the accumulated revealed windows. Thus, temporal BO determines which portions of the spatio-temporal point pattern become available for subsequent spatial intensity inference.

The complete ISBO procedure is summarized in Algorithm~\ref{alg:isbo-summary}. The detailed one-dimensional temporal and two-dimensional spatial implementations are provided in Appendices~\ref{app:ps_code1} and~\ref{app:ps_code2}, respectively.

\begin{algorithm}[t]
\caption{ISBO for point-process data}
\label{alg:isbo-summary}
\KwIn{candidate grid $\mathcal G$, initial event set $\mathcal A_0$, query radius $r$, budget $T$}
\KwOut{query centres $\{x_t\}_{t=1}^{T}$ and updated LGCP posteriors}

\For{$t=0,\ldots,T-1$}{
    Fit the LGCP to $\mathcal A_t$ and obtain
    $\mu_{\eta,t}(x),\sigma_{\eta,t}(x)$ on $\mathcal G$\;

    Select
    \[
        x_{t+1}
        =
        \arg\max_{x\in\mathcal C_t}
        \left\{
        \mu_{\eta,t}(x)+\kappa_{t+1}\sigma_{\eta,t}(x)
        \right\};
    \]

    Reveal events in $\mathcal W(x_{t+1};r)$ and update
    $\mathcal A_{t+1}$ and the masked candidate set $\mathcal C_{t+1}$\;

}
\end{algorithm}

\paragraph{Region-based optimization metrics.}
For synthetic experiments where the true intensity is known, performance is evaluated according to the information contained in each revealed region.
Let
\begin{align}
    \lambda^\star
    &=
    \max_{x\in\mathcal G}\lambda(x),
\end{align}
and define the best intensity contained in the region queried at iteration
$t$ by
\begin{align}
    \lambda_t^{\mathcal W}
    &=
    \max_{x\in\mathcal W(x_t;r)}
    \lambda(x).
    \label{eq:region-query-value}
\end{align}
The instantaneous region regret is
\begin{align}
    r_t
    &=
    \lambda^\star
    -
    \lambda_t^{\mathcal W},
    \label{eq:instantaneous-regret}
\end{align}
and the cumulative regret after $T$ queries is
\begin{align}
    R_T
    &=
    \sum_{t=1}^{T}r_t.
    \label{eq:cumulative-regret}
\end{align}
We further define the simple regret as
\begin{align}
    S_T
    &=
    \lambda^\star
    -
    \max_{1\leq t\leq T}
    \lambda_t^{\mathcal W}.
    \label{eq:simple-regret}
\end{align}

\section{Experiments}

In this section, we evaluate ISBO with spatio-temporal data in 1D temporal, and 2D spatial settings, and compare it against the baseline proposed by \citet{meiBayesianOptGCPMstD}. Details of the data preparation and experimental notes can be found in Appendix~\ref{app:Experiment}.

\subsection{Evaluation on One Dimensional Data}

\subsubsection{1D Temporal Intensity Estimation (Washington, D.C. crime)}\label{sec:1DTint}
We apply ISBO to the May--June 2022 Washington, D.C. firearm-related crime incidents and estimate the temporal intensity using all 342 observed events. Detailed experimental setup and figures are provided in Appendix~\ref{app:1DTint}. Figure~\ref{fig:dc_time_lambda} shows the estimated temporal intensity $\lambda(t)$ over the 60-day horizon. The red curve gives the posterior mean and the gray band the 95\% credible interval. Our end-to-end fit and projection take 2.27s (s for seconds), compared with 35.74s for the baseline. Figure~\ref{fig:dc_time_lambda} (ours) and Figure~\ref{fig:dc_time_lambda_baseline} (baseline) show similar shapes: a peak around day~30, a second rise peaking near day~50, and a drop at the end of June. This agreement on the main structure suggests both methods recover consistent temporal patterns. However the uncertainty differs. The baseline band is tight across most of the domain, which may understate the uncertainty. Moreover, the baseline varies more aggressively and even starts near zero at the beginning, which is inconsistent with the early incidents (zero intensity would imply no expected events). Our SPDE mesh padding and PC priors avoid boundary collapse and yield a smoother, better-calibrated level near the edges.

\begin{figure}[h]
    \centering
    \includegraphics[width=0.45\textwidth]{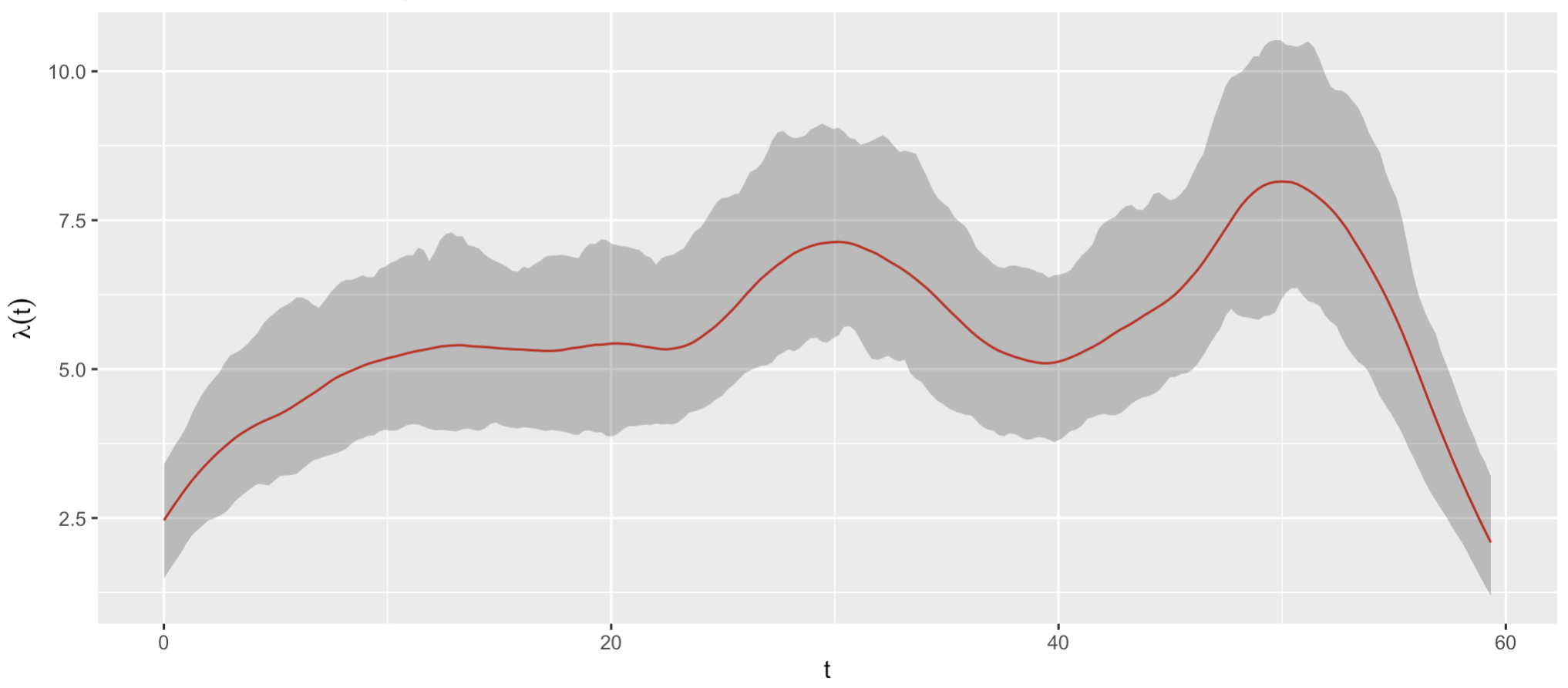}
    \caption{Temporal intensity $\lambda(t)$ (ours)}
    \label{fig:dc_time_lambda}
\end{figure}

\subsubsection{Bayesian Optimization over 1D Synthetic Intensity}
\label{sec:bo_t}

\begin{figure}[t]
    \centering
    \includegraphics[width=\linewidth]
    {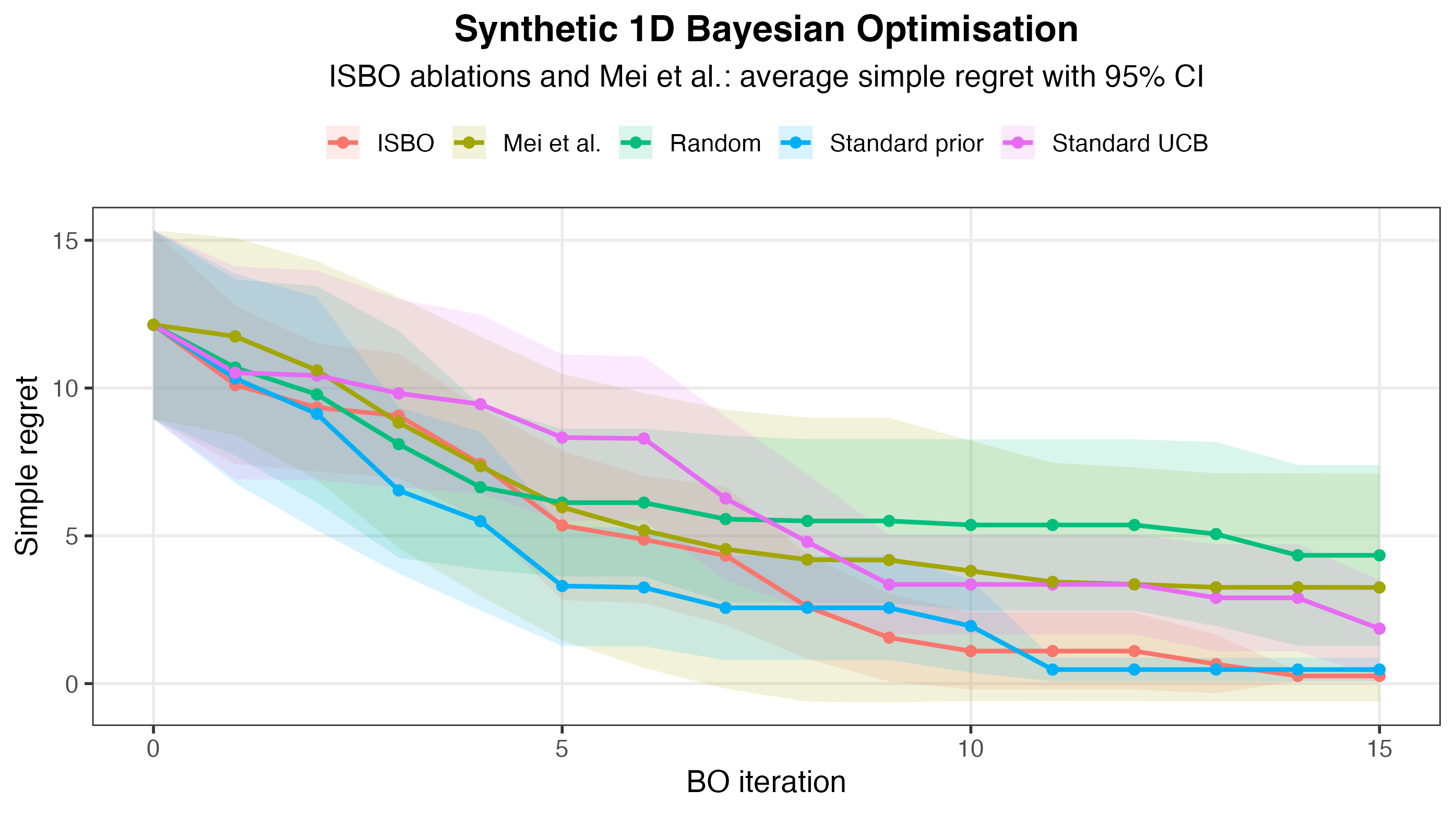}
    \caption{
        Average simple regret for BO over the synthetic
        1D intensity, with 95\% confidence intervals.
    }
    \label{fig:bo_synth_simple_regret}
\end{figure}

We evaluate ISBO on a one-dimensional synthetic point-process dataset with a known ground-truth intensity, allowing optimization performance to be measured directly. The full experimental setup is provided in Appendix~\ref{app:bo_t1d}. We compare the complete ISBO framework with two ablations: \emph{Standard prior}, which replaces the PC prior with a Gaussian prior while keeping the remaining ISBO components unchanged, and \emph{Standard UCB}, which replaces the time-varying exploration schedule with a fixed UCB coefficient. We further include the baseline of \citet{meiBayesianOptGCPMstD} and a random-query policy. All methods share the same initial query locations for each seed. Figure~\ref{fig:bo_synth_simple_regret} reports the average simple regret over 10 seeds with 95\% confidence intervals.

All variants within the ISBO framework achieve lower final simple regret than the baseline of \citet{meiBayesianOptGCPMstD}. The comparison between ISBO and \emph{Standard prior} isolates the effect of PC-prior regularization. Both methods improve rapidly, but ISBO converges more consistently in the later iterations. By iterations 14--15, the mean simple regret of ISBO has reached zero, whereas the \emph{Standard prior} variant, using a Gaussian prior, retains a small positive regret together with a visible confidence interval. This suggests that the PC prior provides more robust regularization of the latent Mat\'ern field under the sparse and sequentially revealed observations encountered during BO.

The comparison with \emph{Standard UCB} isolates the effect of the time-varying exploration schedule. The fixed-UCB variant continues to reduce simple regret throughout the run, but does so more slowly than the complete ISBO method. Its confidence interval is also consistently wider, indicating greater variability across runs. In contrast, the time-varying UCB encourages stronger exploration in early iterations and gradually shifts toward exploitation, leading to faster and more stable localization of the global high-intensity region.

The baseline of \citet{meiBayesianOptGCPMstD} performs substantially worse than all ISBO-based variants and is only moderately better than the random policy in the later iterations. Its confidence interval is also much wider than those of the ISBO variants and is comparable to that of the random policy, indicating substantial run-to-run variability. The ablation study shows that both the PC prior and the time-varying UCB contribute to the improved robustness and peak-localization performance of ISBO.

\subsection{Evaluation on Two Dimensional Data}

\subsubsection{2D Spatial Intensity Estimation (Porto taxi trajectories)}
\label{sec:2DSint}

\begin{figure*}[t]
    \centering

    \begin{minipage}[t]{0.49\textwidth}
        \centering
        \includegraphics[width=\linewidth]
        {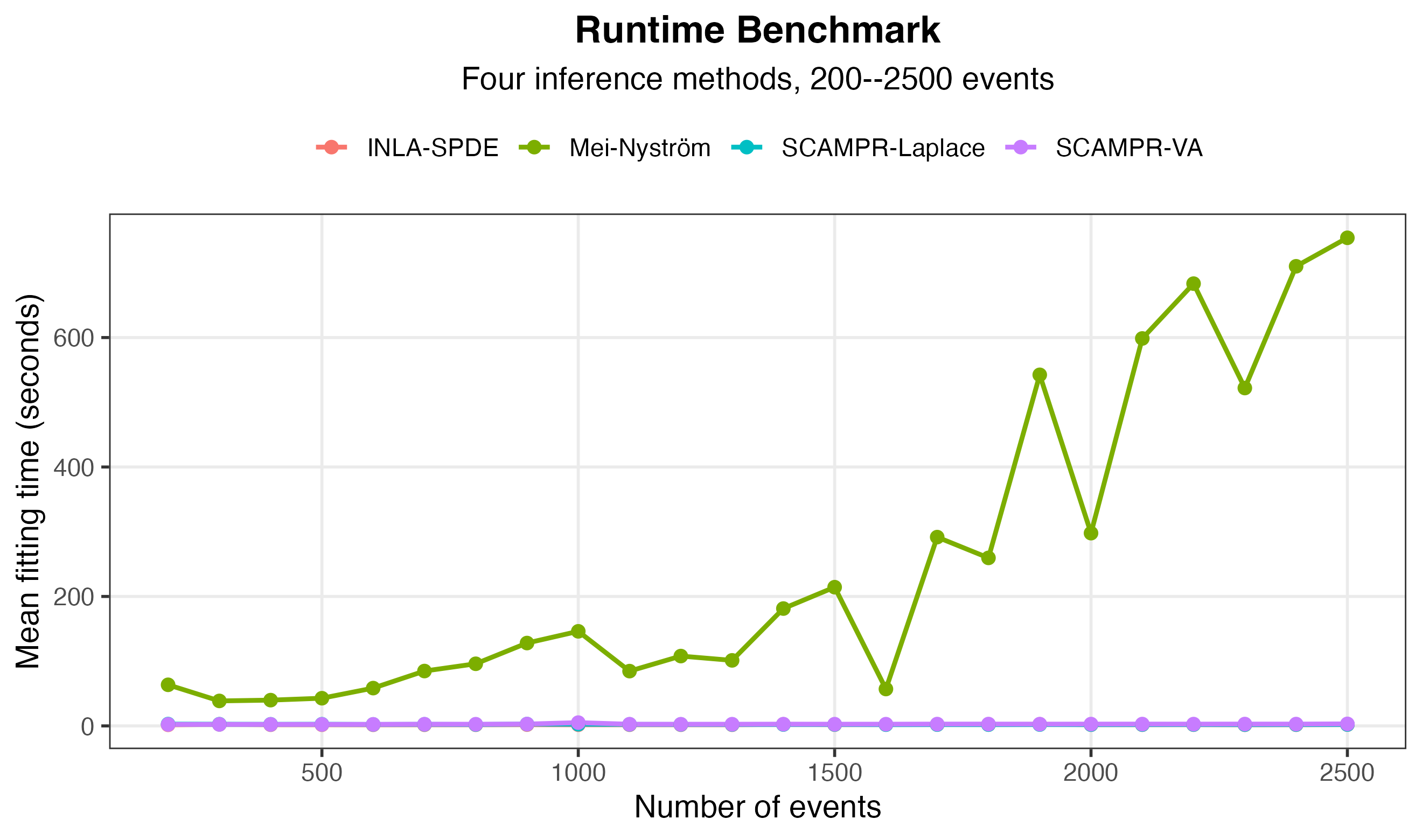}
        \small (a) Runtime benchmark from 200 to 2,500 events.
    \end{minipage}
    \hfill
    \begin{minipage}[t]{0.49\textwidth}
        \centering
        \includegraphics[width=\linewidth]
        {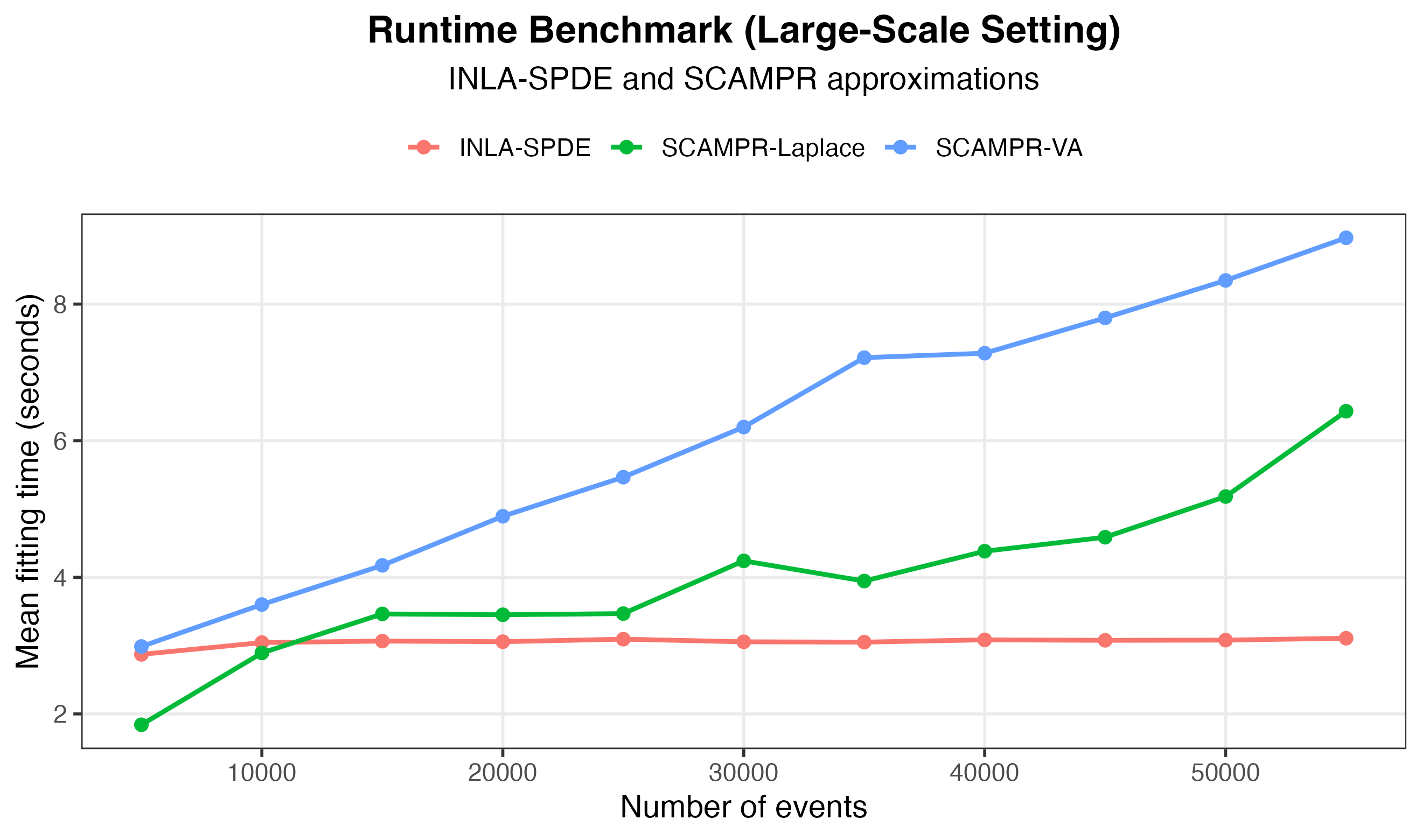}
        \small (b) Large-scale benchmark from 5,000 to 55,000 events.
    \end{minipage}

    \caption{
        Mean fitting time for LGCP inference on local commodity hardware.
        Panel (a) compares INLA--SPDE, the Nystr\"om-based inference procedure
        of \citet{meiBayesianOptGCPMstD}, and the Laplace and variational
        approximations of \citet{dovers2023fast} over 200--2,500 events.
        Panel (b) extends the comparison between INLA--SPDE and the two SCAMPR
        approximations over 5,000 to 55,000 events. 
    }
    \label{fig:runtime_benchmark}
\end{figure*}

We next evaluate spatial LGCP inference on the Porto taxi trajectories, comparing INLA--SPDE with the Nystr\"om-based baseline of \citet{meiBayesianOptGCPMstD} and the SCAMPR approximations of \citet{dovers2023fast}. For the same number of observed events, INLA yields a more detailed and better-resolved intensity surface than the Nystr\"om baseline, as shown in Figures~\ref{fig:taxi_baseline_inten} and \ref{fig:2d_intensity_200_6666}, full experimental details are provided in Appendix~\ref{app:2DSint}. Figure~\ref{fig:runtime_benchmark} compares the computational cost of INLA--SPDE with the Nystr\"om method and the SCAMPR approximations of \citet{dovers2023fast}. In the small-to-moderate regime, Figure~\ref{fig:runtime_benchmark}(a), the Nystr\"om runtime grows rapidly, whereas SCAMPR and INLA--SPDE  are substantially faster, both remains nearly flat.

\citet{dovers2023fast} reported substantial speed-ups over the INLA implementation available at the time, with only negligible loss in statistical efficiency under an appropriate Laplace configuration. Their contribution therefore focused on accelerating LGCP inference while retaining comparable inferential quality with INLA. In contrast, our benchmark with the modern INLA formulation shows that INLA is faster than both SCAMPR approximations over the tested large-scale regime.

The large-scale benchmark in Figure~\ref{fig:runtime_benchmark}(b) shows the same pattern up to 55{,}000 events: the runtime of both SCAMPR variants grows with the number of observations, while INLA--SPDE remains close to constant. These empirical results are consistent with Theorem~\ref{thm:modern-inla-scaling}: under a fixed SPDE mesh and model specification, increasing the number of revealed events does not enlarge the latent system used by modern INLA. ISBO can repeatedly refit the LGCP surrogate at nearly constant computational cost as the revealed dataset grows, providing the scalability required for sequential Bayesian optimization.

\subsubsection{Bayesian Optimization on Spatio- Temporal Data (Washington, D.C. Crime)}
\label{sec:bo_st}

\begin{figure*}[t]
    \centering
    \begin{minipage}[t]{0.49\textwidth}
        \centering
        \includegraphics[width=\linewidth]
        {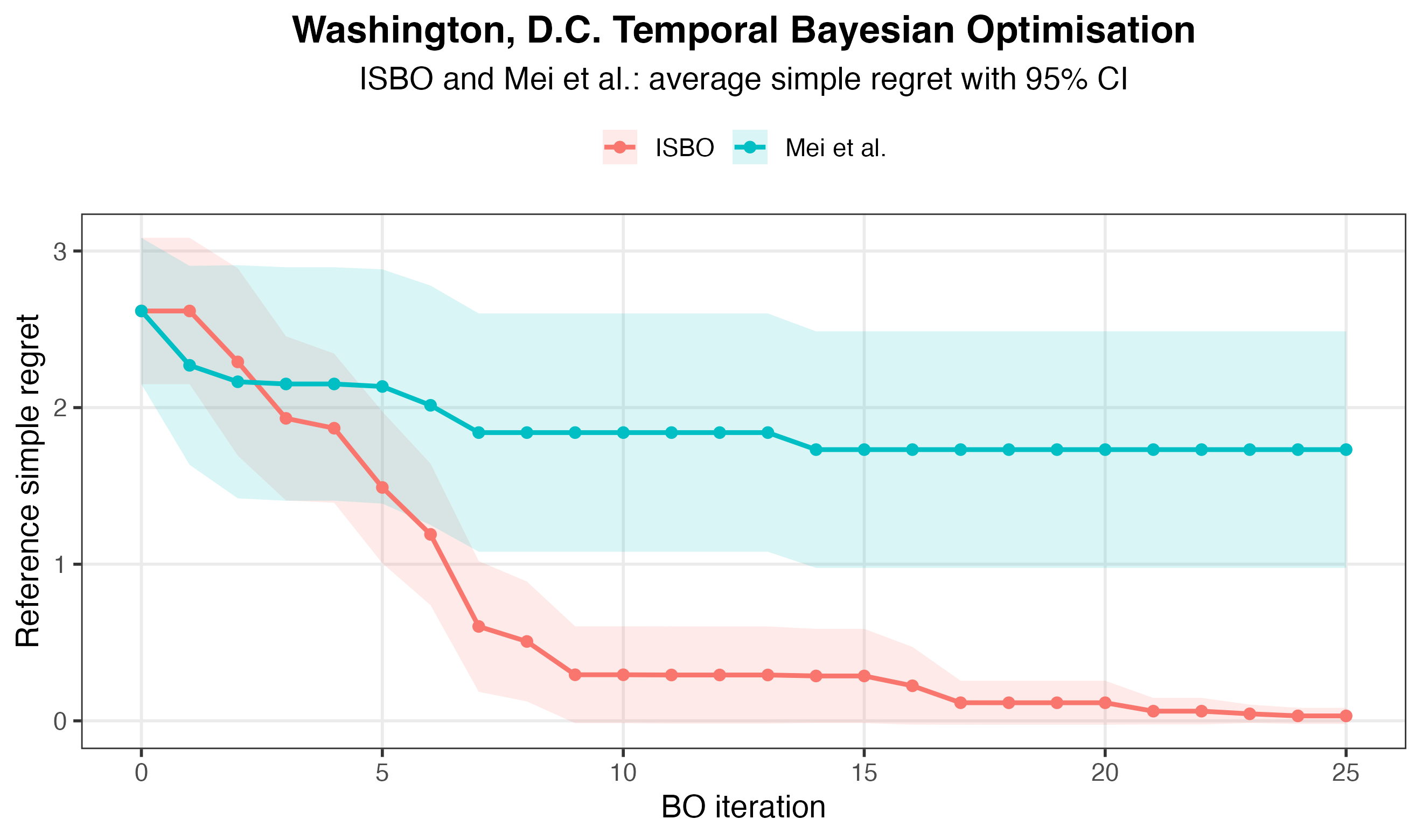}
        \small (a) Average simple regret with 95\% confidence intervals.
    \end{minipage}
    \hfill
    \begin{minipage}[t]{0.49\textwidth}
        \centering
        \includegraphics[width=\linewidth]
        {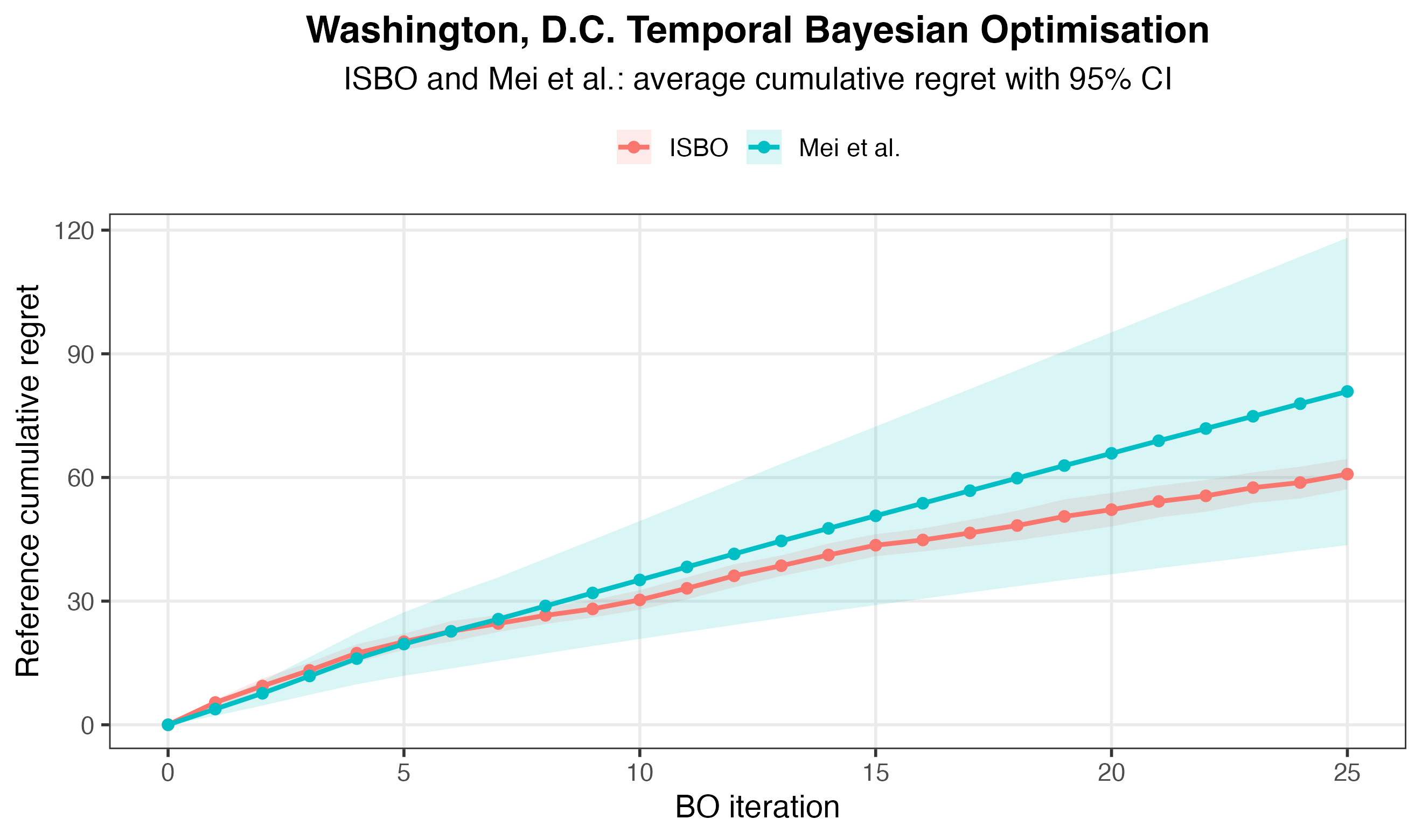}
        \small (b) Average cumulative regret with 95\% confidence intervals.
    \end{minipage}

    \caption{
        Temporal Bayesian optimization on the Washington, D.C.\ crime data,
        comparing ISBO with \citet{meiBayesianOptGCPMstD}. Results are
        averaged across repeated runs, with shaded regions denoting 95\%
        confidence intervals.
    }
    \label{fig:dc_temporal_regret}
\end{figure*}
We finally evaluate ISBO on the Washington, D.C.\ firearm-related crime data in a spatio-temporal setting. To our knowledge, ISBO is the first explicit Cox-process BO framework that directly optimizes a two-dimensional spatial intensity surface. The corresponding spatial BO setup is described in Appendix~\ref{app:bo_spatial}, and the full two-dimensional procedure is summarized in Algorithm~\ref{alg:bo-inla-lgcp-2d}.  In contrast, \citet{meiBayesianOptGCPMstD} performs BO only over temporal intensity and visualizes the revealed spatial events using KDE, rather than fitting and optimizing a spatial Cox-process surrogate.

For a direct comparison, we also benchmark both methods under the same temporal BO setting. Figure~\ref{fig:dc_temporal_regret} shows that ISBO reduces simple regret substantially faster and approaches zero, while the Mei et al.\ baseline plateaus at a higher level. ISBO also achieves lower cumulative regret with markedly narrower confidence intervals in later iterations. Thus, ISBO improves performance in the temporal setting while additionally enabling direct two-dimensional spatial BO.

\section{Discussion}

In this work, we develop ISBO, a scalable Bayesian optimization framework for point-process data based on sequentially revealed event information. Across synthetic and real-world experiments, ISBO consistently improves over existing baselines in optimization performance and computational efficiency, while ablation studies quantify the contribution of its main design components.

From a practical perspective, ISBO enables efficient optimization of latent intensity functions across temporal, spatial, and spatio-temporal settings, while maintaining uncertainty quantification and computational tractability as more events are revealed.

From a methodological perspective, our results show the value of integrating point-process inference directly into the Bayesian optimization loop. By coupling posterior intensity estimation with region-aware acquisition and masking, ISBO provides a coherent framework for adaptive data acquisition under structured event observations. We hope this framework motivates further work on richer acquisition strategies and more general spatio-temporal optimization methods for point-process data.

Beyond the present experiments, ISBO also highlights the potential of INLA, an active and evolving area of research in scalable Bayesian inference, for broader machine-learning problems involving uncertainty quantification and sequential decision making. The framework could be extended to multi-agent reinforcement learning, where scalable Bayesian inference may support hyperparameter tuning, adaptive exploration, and machine decision making \citep{mei2024projection_RL,zhang2024collaborative,erbayat2025lamps,chen2024rgmdt}.

\newpage

\subsection*{AI use statement}

In preparing this manuscript, we used generative AI tools, including ChatGPT, only for limited writing support, such as improving sentence structure, clarity, fluency, and consistency of presentation. We also used AI-assisted features in Visual Studio Code for routine coding support, including code comments and code completion. They made no contribution to any of the core research ideas or scientific conclusions presented in this paper. All core algorithms, experimental pipelines, implementation logic, and scientific contributions were developed by the authors. All AI-assisted text and code were reviewed and revised where necessary. We take full responsibility for the final content of this work.

\bibliography{references}


\clearpage
\appendix
\thispagestyle{empty}

\onecolumn
\aistatstitle{Appendix}

\section{Code and Reproducibility}
\label{app:code}

The code and experimental materials for this paper will be made publicly available upon acceptance.

\section{Preliminaries}
\label{app:pre}

\subsection{Latent Gaussian Models}
\label{app:LGMs}

A latent Gaussian model (LGM) consists of a Gaussian latent field, a conditionally independent likelihood, and hyperpriors for the remaining non-Gaussian parameters. We write $p(\cdot\mid\cdot)$ for conditional densities and define
\begin{align}
    \mathbf{x}
    &=
    \left(
        \alpha_0,\,
        f^{(1)},\ldots,f^{(n_f)},\,
        \beta_1,\ldots,\beta_{n_\beta}
    \right)^\top,
    \label{eq:def-x}
\end{align}
where $\alpha_0$ is an intercept, $f^{(j)}$ denotes a latent random effect, and $\beta_k$ denotes a regression coefficient.

Let $\boldsymbol{\theta}$ denote the non-Gaussian hyperparameters. The latent Gaussian prior is
\begin{align}
    \mathbf{x}\mid\boldsymbol{\theta}
    &\sim
    \mathcal N\!\left(
        \mathbf 0,
        Q(\boldsymbol{\theta})^{-1}
    \right),
    \label{eq:lgm-latent}
\end{align}
where $Q(\boldsymbol{\theta})$ is the precision matrix.

Following the modern INLA formulation, the linear predictors are obtained through
\begin{align}
    \boldsymbol{\eta}
    &=
    A\mathbf{x},
    \label{eq:lgm-eta-modern}
\end{align}
where $A$ is the design or projection matrix. Conditional on
$\boldsymbol{\eta}$, the observations are independent,
\begin{align}
    p(\mathbf y\mid\mathbf x,\boldsymbol{\theta})
    &=
    \prod_{i=1}^{N}
    p\!\left(
        y_i\mid\eta_i,\boldsymbol{\theta}
    \right),
    \label{eq:lgm-like}
\end{align}
with hyperprior
\begin{align}
    \boldsymbol{\theta}
    &\sim
    p(\boldsymbol{\theta}).
    \label{eq:lgm-hyper}
\end{align}

This formulation includes generalized linear and additive models. For example, a linear predictor may be written as
\begin{align}
    \eta_i
    &=
    \alpha_0
    +
    \sum_{k=1}^{n_\beta}
    \beta_k z_{ki}
    +
    \sum_{j=1}^{n_f}
    f^{(j)}(w_{ji}),
    \label{eq:glm-eta}
\end{align}
where $z_{ki}$ are linear covariates and $w_{ji}$ index structured or non-linear effects such as spatial, temporal, i.i.d., or smooth components. The Gaussian structure of these effects, together with sparsity in $Q(\boldsymbol{\theta})$, forms the basis for INLA inference \citep{krainski2018advSPDErINLA}.

\subsection{Gaussian Markov Random Fields}
\label{sec:GMRF}

INLA exploits latent Gaussian Markov random fields (GMRFs) because their conditional independence structure induces sparse precision matrices. In the LGM formulation of \eqref{eq:lgm-latent}, this corresponds to sparsity in $Q(\boldsymbol{\theta})$.

\begin{theorem}[GMRF and Markov Property]
\label{the:gmrf}
A GMRF is a Gaussian vector \(\mathbf{x}=(x_1,\ldots,x_n)^\top\) with covariance \(\boldsymbol{\Sigma}\) and precision \(\mathbf{Q}=\boldsymbol{\Sigma}^{-1}\). The conditional independence are encoded by zeros in \(\mathbf{Q}\) (\(A \perp\!\!\!\perp B\) indicates independence and \(A \perp\!\!\!\perp B \mid C\) conditional independence):
\begin{equation}\label{eq:gmrf-markov}
x_i \,\perp\!\!\!\perp\, x_j \mid \mathbf{x}_{-ij}
\quad\Longleftrightarrow\quad
Q_{ij}=0,
\end{equation}
where \(\mathbf{x}_{-ij}\) denotes all components of \(\mathbf{x}\) except \(x_i\) and \(x_j\).
\end{theorem}

For
\begin{align}
    \mathbf{x}
    &\sim
    \mathcal N(
        \boldsymbol{\mu},
        \boldsymbol{\Sigma}
    ),
    \qquad
    \mathbf{Q}
    =
    \boldsymbol{\Sigma}^{-1},
    \label{eq:gmrf-precision}
\end{align}
the corresponding density is
\begin{align}
    p(\mathbf{x})
    &=
    (2\pi)^{-n/2}
    |\mathbf{Q}|^{1/2}
    \exp\!\left\{
        -\frac{1}{2}
        (\mathbf{x}-\boldsymbol{\mu})^\top
        \mathbf{Q}
        (\mathbf{x}-\boldsymbol{\mu})
    \right\}.
    \label{eq:gmrf-density}
\end{align}

The sparse precision structure enables sparse Cholesky factorization,
\begin{align}
    \mathbf{Q}
    &=
    \mathbf{L}\mathbf{L}^\top,
    \label{eq:gmrf-chol}
\end{align}
where fill-reducing orderings can preserve substantial sparsity in $\mathbf{L}$ \citep{11017730,fattah2026densematricesalikesparse}. For lattice-like structures, typical factorization costs are $\mathcal O(n)$ in one dimension, $\mathcal O(n^{3/2})$ in two dimensions, and $\mathcal O(n^2)$ in three dimensions \citep{rue2005GMRF}. This sparse linear algebra is central to the computational efficiency of INLA.

\subsection{Cox Processes and Log-Gaussian Cox Processes}
\label{sec:lgcp}

A Cox process is a Poisson process with a stochastic intensity. A Log-Gaussian Cox Process (LGCP) specifies
\begin{align}
    \lambda(x)
    &=
    \exp\{\eta(x)\},
    \label{eq:lgcp-general-link}
\end{align}
where $\eta(x)$ is a Gaussian linear predictor and $x$ denotes a generic spatial or temporal location.

Conditional on $\lambda(\cdot)$, the process is an inhomogeneous Poisson process. For a measurable region $\mathcal B$,
\begin{align}
    N(\mathcal B)\mid\lambda
    &\sim
    \operatorname{Poisson}
    \left(
        \int_{\mathcal B}\lambda(x)\,dx
    \right).
\end{align}

For event locations $\{x_i\}_{i=1}^{n}$ observed over a domain $\mathcal D$, the point-process likelihood is
\begin{align}
    p\!\left(
        \{x_i\}_{i=1}^{n}\mid\eta
    \right)
    &\propto
    \exp\!\left\{
        -\int_{\mathcal D}
        \exp\{\eta(x)\}\,dx
    \right\}
    \prod_{i=1}^{n}
    \exp\{\eta(x_i)\}.
    \label{eq:pp-like-general}
\end{align}

For temporal point processes,
\begin{align}
    \mathcal T
    &=
    [t_{\min},t_{\max}],
    \qquad
    L=t_{\max}-t_{\min}>0,
    \label{eq:time-window}
\end{align}
and
\begin{align}
    \lambda(t)
    &=
    \exp\{\eta(t)\}.
    \label{eq:lgcp-time-link}
\end{align}
For spatial point processes, $x=s\in\mathcal D\subset\mathbb R^2$ and
\begin{align}
    \lambda(s)
    &=
    \exp\{\eta(s)\}.
    \label{eq:lgcp-space-link}
\end{align}

\subsection{Laplace Approximation}\label{app:pre_laplace}
We briefly introduce the Laplace approximation here. For a density \(g(x)\) with mode \(\hat x\), a second–order Taylor expansion of \(\log g(x)\) around \(\hat x\) gives
\begin{equation}\label{eq:la-taylor}
\log g(x)\ \approx\ \log g(\hat x)\;-\;\tfrac12\,\bigl[-\ell''(\hat x)\bigr]\,(x-\hat x)^2,
\quad \text{with }\ \ell(x):=\log g(x),
\end{equation}
since \(\ell'(\hat x)=0\) at the mode. Writing the curvature–based variance as
\begin{equation}\label{eq:la-variance}
\hat\sigma^2\ :=\ \bigl[-\ell''(\hat x)\bigr]^{-1},
\end{equation}
we obtain the normal approximation
\begin{equation}\label{eq:la-normal}
g(x)\ \approx\ \text{const}\times \exp\!\Big\{-\tfrac{(x-\hat x)^2}{2\hat\sigma^2}\Big\}
\ \propto\ \mathcal N(x\mid \hat x,\hat\sigma^2).
\end{equation}
The Laplace approximation of \(g(x)\) is approximately normal with mean \(\hat x\) and variance \(\hat\sigma^2\).

\subsection{Spatial Variation and Gaussian Random Fields}
\label{sec:spatial-grf}

For point-referenced spatial data, dependence can be represented through a Gaussian random field (GRF) $U(s)$, where any finite vector
\begin{align}
    \bigl(
        U(s_1),\ldots,U(s_n)
    \bigr)
\end{align}
is jointly Gaussian. A typical spatial linear predictor is
\begin{align}
    \eta(s)
    &=
    \alpha_0
    +
    U(s)
    +
    \mathbf z(s)^\top\boldsymbol{\beta}.
\end{align}

Direct likelihood inference with a GRF builds and factors a \(n\times n\) \emph{dense} covariance matrix, which costs \(\mathcal O(n^3)\) time and \(\mathcal O(n^2)\) memory and does not scale well. A common fix in areal (lattice) data is to model spatial effects with neighborhood structures that lead to sparse precision matrices. In the point–referenced case, we will achieve the same sparsity idea by moving from a covariance view to a precision view: we represent the Matérn field through a SPDE, which yields a Gaussian Markov random field (GMRF) with a sparse precision matrix. This fits naturally inside the INLA framework and provides the needed computational efficiency for our BO setting.

\subsection{The Matérn Covariance}\label{sec:Matérn}

The Matérn family is a standard choice for spatial covariance. For two locations \(s_i\) and \(s_j\) at Euclidean distance \(r=\lVert s_i-s_j\rVert\), the Matérn correlation is
\begin{equation}\label{eq:matern-corr}
\mathrm{Cor}\big(U(s_i),U(s_j)\big)
=\frac{2^{1-\nu}}{\Gamma(\nu)}\,( \kappa r )^{\nu}\,K_{\nu}( \kappa r ),
\end{equation}
where \(\nu>0\) is the smoothness, \(\kappa>0\) is a scale, \(\Gamma(\cdot)\) is the Gamma function, and \(K_{\nu}(\cdot)\) is the modified Bessel function of the second kind. The \emph{covariance} is
\begin{equation}\label{eq:matern-cov}
\mathrm{Cov}\big(U(s_i),U(s_j)\big)=\sigma^2\,\mathrm{Cor}\big(U(s_i),U(s_j)\big),
\end{equation}
with \(\sigma^2\) the marginal variance.

\paragraph{Meaning of the parameters}
\begin{itemize}
\item \(\sigma^2\) sets the overall variability (how large the field values fluctuate around the mean).
\item \(\nu\) controls smoothness, larger \(\nu\) makes the field smoother.
\item \(\kappa\) controls how fast correlation decays with distance. It is often reparametrised by the range \(\rho\), the distance where the correlation is about \(0.13\):
\begin{equation}\label{eq:range-kappa}
\rho \approx \frac{\sqrt{8\nu}}{\kappa}.
\end{equation}
\end{itemize}

\paragraph{PC priors for Matérn hyperparameters}
For the Matérn field we place PC priors on the range and on the marginal standard deviation:
\begin{equation}\label{eq:pc-template}
\Pr\{\rho < \rho_0\}=p_\rho,
\qquad
\Pr\{\sigma > \sigma_0\}=p_\sigma.
\end{equation}
This lets us encode simple beliefs:
\begin{itemize}
\item Range prior: choose \(\rho_0\) as a meaningful fraction of the domain size (for example, half the spatial span), and set \(p_\rho\) to a moderate value (for example, \(0.5\)) to avoid over–short correlation.
\item Standard deviation prior: choose \(\sigma_0\) on the scale of the linear predictor, and set a small tail probability (for example, \(p_\sigma=0.01\)) to discourage very noisy fields.
\end{itemize}

\subsection{The Stochastic Partial Differential Equation Approach}
\label{sec:SPDE}

The SPDE approach of \citet{lindgren2011explicit_GMRF_SPDE} links continuous Mat\'ern Gaussian fields to GMRFs with sparse precision matrices, making Mat\'ern spatial and temporal effects computationally suitable for INLA.

\paragraph{Mat\'ern fields through an SPDE.}
Consider
\begin{align}
    \bigl(\kappa^2-\Delta\bigr)^{\alpha/2}
    \tau U(s)
    &=
    W(s),
    \qquad
    s\in\mathbb R^d,
    \qquad
    \alpha
    =
    \nu+\frac{d}{2},
    \label{eq:spde}
\end{align}
where $\Delta$ is the Laplacian operator, $W(s)$ is Gaussian white noise, $\kappa>0$ and $\tau>0$ are SPDE parameters, and $\nu>0$ controls smoothness. The stationary solution $U(s)$ has a Mat\'ern covariance structure.

\paragraph{Finite-element representation.}
On a bounded domain, the field is approximated as
\begin{align}
    U(s)
    &\approx
    \sum_{k=1}^{m}
    \psi_k(s)w_k,
    \label{eq:fem-expansion}
\end{align}
where $\{\psi_k(s)\}_{k=1}^{m}$ are local piecewise-linear basis functions defined on a mesh and $\mathbf{w}=(w_1,\ldots,w_m)^\top$ are the corresponding Gaussian weights.

The FEM construction gives
\begin{align}
    \mathbf{w}\mid\boldsymbol{\theta}
    &\sim
    \mathcal N\!\left(
        \mathbf 0,
        Q_{\mathrm{SPDE}}(\boldsymbol{\theta})^{-1}
    \right),
\end{align}
where $Q_{\mathrm{SPDE}}(\boldsymbol{\theta})$ is sparse because the basis functions have local support. This connects the continuous Mat\'ern field to the GMRF framework in \cref{sec:GMRF}. The mesh may also be refined locally and adapted to irregular domain boundaries.

For temporal data, the same construction applies with $d=1$ and a one-dimensional finite-element mesh.

\section{Low-rank Correction Using Variational Bayes}\label{app:LRCorVB}

In this appendix, we summarize the modern INLA formulation together with the low-rank variational Bayes (VB) correction proposed by \citet{van2023new}, focusing on the main results required for ISBO.

\subsection{Reformulation of the Latent Gaussian Model}

We consider a latent Gaussian model in which the latent field is defined as
\begin{equation}
\mathbf{x}
= (\alpha_0, f^{(1)}, \ldots, f^{(n_f)}, \beta_1, \ldots, \beta_{n_\beta})^\top,
\end{equation}
collecting the intercept, latent random effects, and regression coefficients, but excluding the linear predictors $\boldsymbol{\eta}$. The latent field is assigned a Gaussian prior
\begin{equation}
\mathbf{x}\mid\boldsymbol{\theta}
\sim \mathcal{N}\bigl(\mathbf{0}, Q(\boldsymbol{\theta})^{-1}\bigr),
\end{equation}
where $Q(\boldsymbol{\theta})$ denotes the sparse precision matrix parameterized by the hyperparameters $\boldsymbol{\theta}$.

In the modern INLA formulation, the linear predictors are no longer treated as explicit components of the latent field. Instead, they are defined as a linear transformation of $\mathbf{x}$,
\begin{equation}
\boldsymbol{\eta} = A\mathbf{x},
\end{equation}
where $A$ is a sparse design matrix linking the latent field to the observations. By removing the $n$ linear predictors from the latent field, this formulation substantially improves both scalability and numerical stability, particularly in data-rich settings.

\subsection{Approximating $\tilde p(\boldsymbol{\theta}\mid\mathbf{y})$}

Under the modern formulation, the joint density of the latent field, the hyperparameters, and the observations is given by

\begin{equation}
p(\mathbf{x},\boldsymbol{\theta}\mid\mathbf{y})
\propto
p(\boldsymbol{\theta})\,
p(\mathbf{x}\mid\boldsymbol{\theta})
\prod_{i=1}^n p\bigl(y_i \mid (A\mathbf{x})_i,\boldsymbol{\theta}\bigr),
\end{equation}

Following the Laplace approximation strategy, the posterior of the hyperparameters is approximated as
\begin{equation}
\tilde p(\boldsymbol{\theta}\mid\mathbf{y})
\propto
\left.
\frac{p(\mathbf{x},\boldsymbol{\theta}\mid\mathbf{y})}
     {p_G(\mathbf{x}\mid\boldsymbol{\theta},\mathbf{y})}
\right|_{\mathbf{x}=\boldsymbol{\mu}(\boldsymbol{\theta})},
\end{equation}
where $p_G(\mathbf{x}\mid\boldsymbol{\theta},\mathbf{y})$ denotes a Gaussian approximation to the conditional posterior $p(\mathbf{x}\mid\boldsymbol{\theta},\mathbf{y})$, evaluated at its mode $\boldsymbol{\mu}(\boldsymbol{\theta})$.

To construct $p_G(\mathbf{x}\mid\boldsymbol{\theta},\mathbf{y})$ for fixed $\boldsymbol{\theta}$, a second-order Taylor expansion of the log-likelihood around the mode $\boldsymbol{\mu}(\boldsymbol{\theta})$ is employed, yielding the Gaussian approximation
\begin{equation}
\mathbf{x}\mid\boldsymbol{\theta},\mathbf{y}
\approx
\mathcal{N}\bigl(\boldsymbol{\mu}(\boldsymbol{\theta}),
Q_X(\boldsymbol{\theta})^{-1}\bigr),
\end{equation}
with precision matrix
\begin{equation}
Q_X(\boldsymbol{\theta})
=
Q(\boldsymbol{\theta})
+
A^\top D(\boldsymbol{\theta}) A,
\end{equation}
where $D(\boldsymbol{\theta})$ is a diagonal matrix with $n$ entries.

\subsection{Approximating $p(x_j\mid\boldsymbol{\theta},\mathbf{y})$}

Conditionally on $\boldsymbol{\theta}$, the Gaussian approximation implies
\begin{equation}
x_j\mid\boldsymbol{\theta},\mathbf{y}
\sim
\mathcal{N}\bigl(\mu_j(\boldsymbol{\theta}),
(Q_X(\boldsymbol{\theta})^{-1})_{jj}\bigr),
\end{equation}
where $\mu_j(\boldsymbol{\theta})$ is the $j$th component of
$\boldsymbol{\mu}(\boldsymbol{\theta})$.
The marginal posterior is obtained by numerical integration over
$\boldsymbol{\theta}$,
\begin{equation}
\tilde p(x_j\mid\mathbf{y})
\approx
\sum_{k=1}^K
p_G(x_j\mid\boldsymbol{\theta}_k,\mathbf{y})\,
\tilde p(\boldsymbol{\theta}_k\mid\mathbf{y})\,\delta_k.
\end{equation}

\subsection{Marginal Posteriors of the Linear Predictors} \label{app:MP_LP}

In contrast to the classic INLA formulation, where the linear predictors are included as explicit components of the latent field and their marginal posteriors are obtained automatically as part of the inference, the modern formulation removes the linear predictors from the latent field. As a consequence, marginal posteriors of the linear predictors are not computed directly and must instead be recovered after approximating $\tilde p(\boldsymbol{\theta}\mid\mathbf{y})$ and $p(\mathbf{x}\mid\boldsymbol{\theta},\mathbf{y})$.

Since $\boldsymbol{\eta}=A\mathbf{x}$, the conditional posterior of $\eta_i$ follows from the Gaussian approximation as
\begin{equation}
\eta_i\mid\boldsymbol{\theta},\mathbf{y}
\sim
\mathcal{N}\bigl(\mu_i(\boldsymbol{\theta}),
\sigma_i^2(\boldsymbol{\theta})\bigr),
\end{equation}

with

\begin{equation}
\mu_i(\boldsymbol{\theta}) = (A\boldsymbol{\mu}(\boldsymbol{\theta}))_i,
\qquad
\sigma_i^2(\boldsymbol{\theta})
= (A Q_X(\boldsymbol{\theta})^{-1} A^\top)_{ii}.
\end{equation}

The diagonal elements of $A Q_X(\boldsymbol{\theta})^{-1} A^\top$ are computed efficiently using the sparse selected inverse of $Q_X(\boldsymbol{\theta})$.

Finally, marginalization over the hyperparameters yields
\begin{equation}
\tilde p(\eta_i\mid\mathbf{y})
\approx
\sum_{k=1}^K
p_G(\eta_i\mid\boldsymbol{\theta}_k,\mathbf{y})\,
\tilde p(\boldsymbol{\theta}_k\mid\mathbf{y})\,\delta_k,
\end{equation}
where $\{\boldsymbol{\theta}_k\}_{k=1}^K$ and $\{\delta_k\}$ denote the numerical integration points and corresponding weights.

\subsection{Low-rank Variational Bayes Correction}

For non-Gaussian likelihoods, the Gaussian approximation $p_G(\mathbf{x}\mid\boldsymbol{\theta},\mathbf{y})$ may yield biased posterior means, even when the covariance structure is well captured. To improve the accuracy of posterior means without reverting to a full nested Laplace approximation, a low-rank variational Bayes (VB) correction to the conditional mean is applied \citep{van2024low}.

The correction is introduced within a restricted Gaussian variational family that retains the precision matrix $Q_X(\boldsymbol{\theta})$ and modifies only the mean. Specifically, the corrected mean is defined as
\begin{equation}
\boldsymbol{\mu}^*(\boldsymbol{\theta})
=
\boldsymbol{\mu}(\boldsymbol{\theta}) + M\boldsymbol{\lambda},
\end{equation}
where $\boldsymbol{\lambda}\in\mathbb{R}^p$, and
$M$ is a propagation matrix constructed from selected columns of
$Q_X(\boldsymbol{\theta})^{-1}$.

The correction vector $\boldsymbol{\lambda}$ is obtained by minimizing a variational objective consisting of the expected negative log-likelihood under the Gaussian distribution $\mathcal{N}(\boldsymbol{\mu}^*(\boldsymbol{\theta}),
Q_X(\boldsymbol{\theta})^{-1})$ together with a quadratic penalty induced by the prior precision $Q(\boldsymbol{\theta})$. This optimization is carried out efficiently using a second-order Taylor approximation around $\boldsymbol{\lambda}=\mathbf{0}$, with expectations of the log-likelihood evaluated via Gauss--Hermite quadrature. The resulting corrected mean achieves accuracy comparable to a full nested Laplace approximation while incurring only a negligible additional computational cost.

The improved mean $\boldsymbol{\mu}^*(\boldsymbol{\theta})$ can then be used in the Gaussian approximation for the linear predictors: 
$\boldsymbol{\theta}$,
\begin{equation}
\eta_j\mid\boldsymbol{\theta},\mathbf{y}
\sim
\mathcal{N}\bigl(\mu_j(\boldsymbol{\theta}),
\sigma_j^2(\boldsymbol{\theta})\bigr),
\qquad
\mu_j(\boldsymbol{\theta}) = (A\boldsymbol{\mu}^*(\boldsymbol{\theta}))_j,
\end{equation}
with $\sigma_j^2(\boldsymbol{\theta})$ defined as in
Section~\ref{app:MP_LP}.
Finally, marginalization over the hyperparameters yields
\begin{equation}
\tilde p(\eta_j\mid\mathbf{y})
\approx
\sum_{k=1}^K
p_G(\eta_j\mid\boldsymbol{\theta}_k,\mathbf{y})\,
\tilde p(\boldsymbol{\theta}_k\mid\mathbf{y})\,\delta_k,
\end{equation}
which provides improved marginal posteriors of the linear predictors while retaining the computational efficiency of the modern INLA framework.

For more details, see \citet{van2023new} and \citet{van2024low}.

\section{Key properties and usage of PC priors}\label{app:PC_prior}

Below we summarise the key properties of PC priors \citep{PCprior} and explain why they are particularly useful in our INLA–SPDE approach BO framework.

\begin{itemize}

\item \textbf{Invariance.} Invariant under one-to-one reparameterization.
\item \textbf{Jeffreys link.} Near the base \(\theta_0\), \(\pi(\theta)\propto \sqrt{I(\theta_0)}\), matching Jeffreys’ rule locally.
\item \textbf{Occam’s razor.} Shrinks toward the base model, only adds complexity when data support it.
\item \textbf{Robustness.} Stable with weak/noisy data, controlled tails reduce overfitting.
\item \textbf{Default use.} These points make PC priors good default hyperpriors in the LGM block \eqref{eq:lgm-hyper}.
\item \textbf{Our stance for BO.} Unlike the baseline that ignores priors, we stress that priors matter in early BO when data are scarce. With INLA we also get \(\tilde{\pi}(\boldsymbol{\theta}\mid \mathbf{y})\) for hyperparameters, which helps practical tuning and checks.
\end{itemize}

\section{Proof of Theorem~\ref{thm:latent-scale-peak}}
\label{app:latent-scale-proof}

\begin{proof}
We prove the two statements separately.

\paragraph{Part I: equivalence of the maximizers.}
Recall that the LGCP intensity is defined by
\begin{align}
    \lambda(x)
    =
    \exp\{\eta(x)\},
    \qquad
    x\in\mathcal G.
\end{align}
The exponential function
\begin{align}
    h(z)=\exp\{z\}
\end{align}
is strictly increasing on $\mathbb R$. Hence, for any
$x_1,x_2\in\mathcal G$,
\begin{align}
    \eta(x_1)>\eta(x_2)
    \quad\Longleftrightarrow\quad
    \exp\{\eta(x_1)\}>
    \exp\{\eta(x_2)\}.
\end{align}
Equivalently,
\begin{align}
    \eta(x_1)>\eta(x_2)
    \quad\Longleftrightarrow\quad
    \lambda(x_1)>\lambda(x_2).
\end{align}

Thus, the exponential transformation preserves the complete ordering of the candidate locations in $\mathcal G$. In particular, if
$x^\star\in\mathcal G$ maximizes $\eta(x)$, then for every
$x\in\mathcal G$,
\begin{align}
    \eta(x^\star)\geq\eta(x),
\end{align}
and therefore
\begin{align}
    \lambda(x^\star)
    =
    \exp\{\eta(x^\star)\}
    \geq
    \exp\{\eta(x)\}
    =
    \lambda(x).
\end{align}
Hence, $x^\star$ also maximizes $\lambda(x)$. The converse follows by the same strict monotonicity argument. Therefore,
\begin{align}
    \arg\max_{x\in\mathcal G}\eta(x)
    =
    \arg\max_{x\in\mathcal G}\lambda(x).
\end{align}

This result also allows for multiple maximizers. If two locations satisfy
\begin{align}
    \eta(x_1)=\eta(x_2),
\end{align}
then
\begin{align}
    \lambda(x_1)
    =
    \exp\{\eta(x_1)\}
    =
    \exp\{\eta(x_2)\}
    =
    \lambda(x_2),
\end{align}
so the entire set of maximizers is preserved by the transformation.

\paragraph{Part II: posterior standard deviation on the intensity scale.}
Suppose that, after conditioning on the events revealed up to BO iteration
$t$,
\begin{align}
    \eta(x)\mid\mathcal A_t
    \approx
    \mathcal N\!\left(
        \mu_{\eta,t}(x),
        \sigma_{\eta,t}^{2}(x)
    \right).
\end{align}
For brevity, write
\begin{align}
    \mu
    =
    \mu_{\eta,t}(x),
    \qquad
    \sigma^2
    =
    \sigma_{\eta,t}^{2}(x).
\end{align}
Since
\begin{align}
    \lambda(x)
    =
    \exp\{\eta(x)\},
\end{align}
the Gaussian approximation for $\eta(x)\mid\mathcal A_t$ induces an approximately log-normal posterior for $\lambda(x)\mid\mathcal A_t$.

To obtain its moments, note that if
\begin{align}
    Z\sim\mathcal N(\mu,\sigma^2),
\end{align}
then the moment-generating function of $Z$ is
\begin{align}
    \mathbb E[\exp\{aZ\}]
    =
    \exp\!\left\{
        a\mu+\frac{a^2\sigma^2}{2}
    \right\}
\end{align}
for any $a\in\mathbb R$. Taking $a=1$ gives
\begin{align}
    \mathbb E[
        \lambda(x)
        \mid
        \mathcal A_t
    ]
    &=
    \mathbb E[
        \exp\{\eta(x)\}
        \mid
        \mathcal A_t
    ]
    \\
    &\approx
    \exp\!\left\{
        \mu
        +
        \frac{\sigma^2}{2}
    \right\}.
    \label{eq:appendix-lambda-mean}
\end{align}

Similarly, taking $a=2$ gives the second moment
\begin{align}
    \mathbb E[
        \lambda(x)^2
        \mid
        \mathcal A_t
    ]
    &=
    \mathbb E[
        \exp\{2\eta(x)\}
        \mid
        \mathcal A_t
    ]
    \\
    &\approx
    \exp\!\left\{
        2\mu
        +
        2\sigma^2
    \right\}.
    \label{eq:appendix-lambda-second-moment}
\end{align}

Therefore,
\begin{align}
    \operatorname{Var}
    \left[
        \lambda(x)
        \mid
        \mathcal A_t
    \right]
    &=
    \mathbb E[
        \lambda(x)^2
        \mid
        \mathcal A_t
    ]
    -
    \left(
        \mathbb E[
            \lambda(x)
            \mid
            \mathcal A_t
        ]
    \right)^2
    \\
    &\approx
    \exp\!\left\{
        2\mu+2\sigma^2
    \right\}
    -
    \exp\!\left\{
        2\mu+\sigma^2
    \right\}
    \\
    &=
    \exp\!\left\{
        2\mu+\sigma^2
    \right\}
    \left(
        \exp\{\sigma^2\}-1
    \right).
    \label{eq:appendix-lambda-var}
\end{align}

Taking the positive square root yields
\begin{align}
    \operatorname{sd}
    \left[
        \lambda(x)
        \mid
        \mathcal A_t
    \right]
    &\approx
    \exp\!\left\{
        \mu+\frac{\sigma^2}{2}
    \right\}
    \sqrt{
        \exp\{\sigma^2\}-1
    }.
\end{align}
Substituting
$\mu=\mu_{\eta,t}(x)$ and
$\sigma^2=\sigma_{\eta,t}^{2}(x)$ gives
\begin{align}
    \operatorname{sd}
    \left[
        \lambda(x)
        \mid
        \mathcal A_t
    \right]
    &\approx
    \exp\!\left\{
        \mu_{\eta,t}(x)
        +
        \frac{1}{2}
        \sigma_{\eta,t}^{2}(x)
    \right\}
    \sqrt{
        \exp\!\left\{
            \sigma_{\eta,t}^{2}(x)
        \right\}
        -1
    },
\end{align}
which is exactly the expression in \eqref{eq:lambda-posterior-sd} under the Gaussian posterior approximation.

The first part shows that optimizing the latent log-intensity and optimizing the intensity target the same locations. The second part shows that the exponential transformation maps Gaussian uncertainty in $\eta(x)$ to log-normal uncertainty in $\lambda(x)$, whose scale increases nonlinearly with $\sigma_{\eta,t}^{2}(x)$. This motivates defining the UCB-style acquisition directly on the latent Gaussian scale, as in \eqref{eq:isbo-acquisition}.
\end{proof}

\section{Proof of Theorem~\ref{thm:modern-inla-scaling}}
\label{app:modern-inla-scaling-proof}

\begin{proof}
Let $n_t$ denote the number of events revealed by ISBO up to iteration $t$. Under the modern INLA formulation, the latent field is defined in
\eqref{eq:modern-latent-field} as
\begin{align}
    \mathbf{x}
    =
    \left(
        \alpha_0,\,
        \mathbf w^\top,\,
        \boldsymbol{\beta}^\top
    \right)^\top .
\end{align}
Let $m$ denote the number of SPDE mesh coefficients and $n_\beta$ the number of regression coefficients. Since both the SPDE mesh and the model specification are fixed throughout BO,
\begin{align}
    p_x
    :=
    \dim(\mathbf{x})
    =
    1+m+n_\beta
\end{align}
is constant and, crucially, does not depend on $n_t$.

At iteration $t$, the linear predictors are represented by
\begin{align}
    \boldsymbol{\eta}_t
    =
    A_t\mathbf{x},
\end{align}
where the number of rows of $A_t$ may increase as additional events are revealed. However, the number of columns of $A_t$ is always $p_x$, because each row projects onto the same fixed latent field.

For fixed $\boldsymbol{\theta}$, the Gaussian approximation used by modern INLA has precision
\begin{align}
    Q_X^{(t)}(\boldsymbol{\theta})
    =
    Q(\boldsymbol{\theta})
    +
    A_t^\top
    D_t(\boldsymbol{\theta})
    A_t .
    \label{eq:app-modern-QXt}
\end{align}
Although $A_t$ and $D_t(\boldsymbol{\theta})$ depend on the observations available at iteration $t$, the matrix
\begin{align}
    A_t^\top
    D_t(\boldsymbol{\theta})
    A_t
\end{align}
always has dimension $p_x\times p_x$. Therefore, $Q_X^{(t)}(\boldsymbol{\theta})$ has the same dimension at every BO iteration, independently of $n_t$. Moreover, the sparsity structure is controlled by the fixed SPDE mesh. The finite-element basis functions have local support, so each observation location involves only the basis functions associated with its local mesh element. Hence, adding further observations changes the numerical entries of $A_t^\top D_t A_t$, but does not introduce an observation-sized latent block. The sparse precision system solved by INLA therefore remains defined on the same mesh-based graph throughout the BO procedure. This differs from the old formulations in which every linear predictor $\eta_i$ is included explicitly in the latent field. In such a formulation, the latent dimension grows with the number of observations, whereas in the modern formulation the predictors are recovered through $\boldsymbol{\eta}_t=A_t\mathbf{x}$ without enlarging $\mathbf{x}$.

Thus the sparse Cholesky factorization and associated linear solves, which constitute the dominant computational operations in the modern INLA fit, are performed on a matrix of fixed dimension and fixed mesh-controlled sparsity structure at every BO iteration. Their computational complexity is therefore independent of the increasing number of revealed events.

We conclude that, under a fixed SPDE mesh and fixed model specification, the dominant INLA refitting cost remains constant across ISBO iterations as the revealed dataset grows. Operations such as evaluating individual likelihood contributions and assembling $A_t^\top D_t A_t$ may still depend on $n_t$, but these do not increase the dimension of the principal sparse system. This establishes the stated constant-refitting property of modern INLA in the sense relevant to ISBO.
\end{proof}

\section{Experiment Setups}\label{app:Experiment}

\subsection{Data Preparations}
In this paper, we use one synthetic dataset and two real-world datasets: the 2022 Washington, D.C. crime incidents \citep{dcgov2022crime} and the 2013-2014 Porto taxi trajectories \citep{kaggle_taxi}. To enable a fair comparison, we follow the baseline’s data construction and preprocessing whenever possible. Since the original experimental data were unavailable, our datasets differ slightly but match the overall distribution, so results are unaffected. We also cloned the baseline’s public code and reran all experiments using our data for comparison. Below, we briefly describe how we construct and preprocess these datasets.

\paragraph{Synthetic 1D data.}
The baseline did not release data and generated sequences without a fixed seed, so we recreated its setup. We generate a reproducible one-dimensional inhomogeneous Poisson process on the interval $[0,100]$ using a fixed random seed. The ground-truth intensity is constructed directly as a smooth multi-modal function consisting of a positive baseline, three Gaussian-shaped peaks, and a small sinusoidal component, with one dominant peak. We evaluate this intensity on a uniform grid of 1001 points and retain the resulting curve as the known ground truth for subsequent optimisation and regret evaluation. Point-process events are then simulated using the thinning algorithm \citep{thinning_lewis1979simulation}: candidate events are first drawn from a homogeneous Poisson process with dominating rate $\lambda_{\max}=\max_t \lambda(t)$ and uniformly distributed over the observation interval, after which a candidate at time $t$ is accepted with probability $\lambda(t)/\lambda_{\max}$. With the fixed seed used in our experiments, this procedure yields 596 point-process events. The construction provides both the event observations used by the optimisation algorithms and an exact reference intensity against which peak recovery and regret can be evaluated.

\paragraph{Washington, D.C. crime (2022).}
We use the publicly available 2022 spatio-temporal crime dataset from Washington, D.C., USA. We filter to firearm-related violent incidents, yielding 342 events (the baseline reports 343). For our experiments we take the May–June subset. We normalize time to continuous days since May 1 (including fractional days), then sort by time and keep latitude/longitude aligned. Since the true underlying intensity is unavailable for this real-world dataset, we construct a reference intensity for quantitative evaluation, details are provided in Appendix~\ref{app:bo_spatial}.

\paragraph{Porto taxi trajectories (2013--2014).}
We use the Porto taxi trajectory dataset from 2013--2014. After removing anomalous records, we retain only trips with \textbf{\texttt{CALL\_TYPE}~C} (A = central dispatch, B = stand pickup, C = random street hail). For each trajectory, we take the first GPS point as the trip starting point and filter the points within the spatial window $[41.15,\,41.18]\times[-8.63,\,-8.60]$. Timestamps are converted to continuous hours since \texttt{2013-07-01 00:00 UTC}. The modeling table contains origin latitude/longitude and the associated time vector, plus an auxiliary placeholder ``intensity'' column for downstream models. The raw dataset contains 106{,}732 records. The baseline paper uses only about two hundred trips, but the exact subset is not available. We therefore sample the required number of trips from the full data using a fixed random seed according to different experiment setup.

\subsection {Experimental Note on IPS usage.}
In our experiments, we did \emph{not} construct a separate spatial IPS, spatial LGCP fits always
integrate over the entire $\mathcal D$ using the domain–mesh weights (i.e., \eqref{eq:ips-space-default}). While along the temporal axis we \emph{did} build a 1D integration mesh (see Figure~\ref{fig:DC_t_mesh}) $\mathcal M^{t}_{\mathrm{ips}}$ with a $5\%$ padding around $[t_{\min},t_{\max}]$. We do this for a simple reason. In our data, edge effects are larger in time than in space. After the usual padding, spatial edge effects are small. So we set an IPS only on $\mathcal T$. This choice depends purely on the specific task: if a problem shows strong spatial edge effects, one can add a spatial IPS as well.

\subsection{Experimental Note on Query Radius.}
The query radius \(r\) controls the size of the region revealed at each BO step and therefore must be chosen carefully. If \(r\) is too large relative to the domain, a single query may reveal most or even all of the data, making the optimization problem effectively trivial. Such a setting is neither informative for evaluating sequential Bayesian optimization nor realistic in applications, where each query typically provides only local information. We therefore restrict \(r\) to a moderate range so that each iteration reveals only a limited portion of the domain, preserving the intended sequential decision-making setting. In practice, \(r\) should be interpreted as the spatial or temporal extent of a feasible observation window and selected according to the resolution and scale of the application.

\subsection{Experimental Note on Computing Resources}

All experiments, including ISBO and the competing baselines, are executed locally on a MacBook Air with an Apple M2 processor and 16\,GB RAM. We do not use HPC resources for any experiment. Our R implementation uses R~4.5.1 with \texttt{INLA}~25.06.07 and \texttt{inlabru}~2.13.0, while the Python baselines are also executed locally on the same machine. This choice is motivated by the intended use of ISBO in resource-constrained settings, where access to HPC infrastructure or even a reliable Internet connection may not be available. Such constraints arise, for example, in unmanned aerial vehicle (UAV) applications and related mobile sensing platforms \citep{UAV1gupta2015survey,UAV2fahlstrom2022introduction,UAV3nex2014uav}. Running all methods on commodity local hardware therefore provides a stricter and more practically relevant evaluation of computational scalability. In particular, our aim is to demonstrate that ISBO is sufficiently fast and memory-efficient for repeated Bayesian optimization without relying on specialized computing infrastructure.

\section{1D Temporal Intensity Estimation (Washington, D.C. crime)}\label{app:1DTint}
We map time to a continuous 60-day axis, i.e., $t\in[0,60]$ with $t=0$ on May~1st and $t=60$ on June~30th. Figure~\ref{fig:DC_data_t} illustrates the data: black vertical ticks along the x-axis mark each incident time, and a histogram with 20 equal-width bins shows the counts per bin.

To fit the LGCP on the 1D time axis in a stable and efficient way, we need to build a 1D domain mesh $\mathcal M_{\mathrm{dom}}^{t}$ over the observed window $\mathcal T=[t_{\min},t_{\max}]$ and discretize the latent field with piecewise-linear basis functions centered at the mesh nodes, yielding a sparse GMRF for fast INLA inference. Concretely, we pad $\mathcal T$ by 10\% on each side to define $\mathcal M_{\mathrm{dom}}^{t}$ (hereafter \texttt{mesh1D}), and construct a separate integration mesh $\mathcal M_{\mathrm{ips}}^{t}$ (hereafter \texttt{ips}) with a 5\% padding, numerical quadrature is performed only inside $\mathcal M_{\mathrm{ips}}^{t}$ to mitigate boundary effects. We use 300 evenly spaced nodes to balance resolution and computational cost. Figure~\ref{fig:DC_t_mesh} visualizes the padded meshes and sampling structure: nodes (bottom), observed events (middle), and integration points (top).

\begin{figure}[h]
    \centering
    \includegraphics[width=0.5\textwidth]{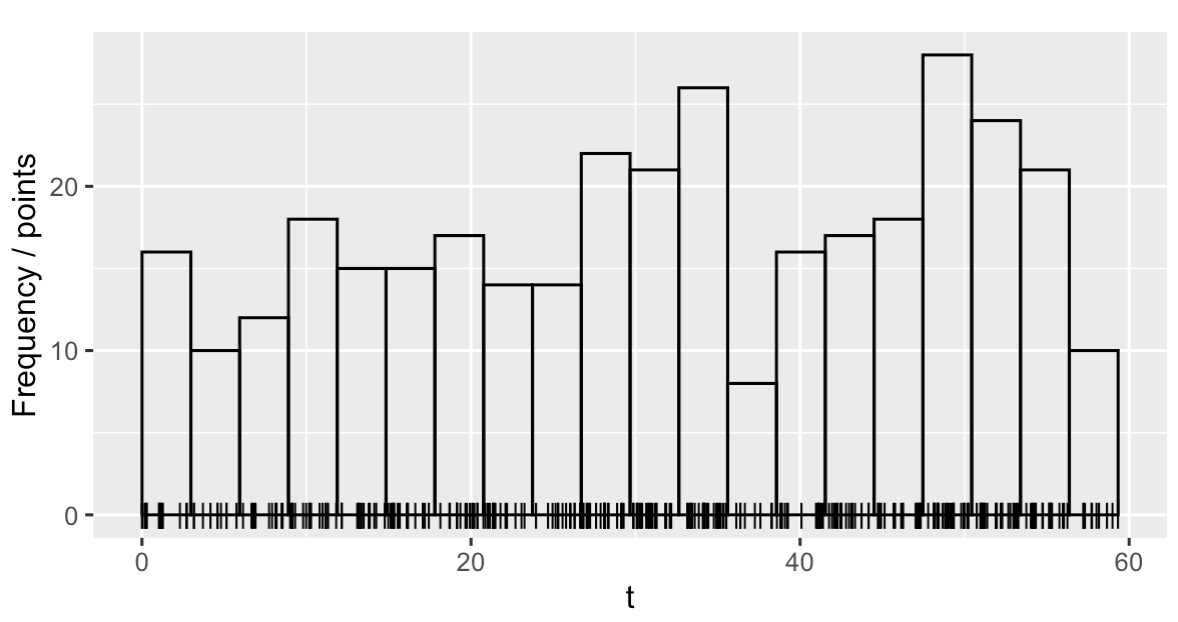}
    \caption{Observed firearm-related crime incidents time at Washington, D.C. between May and June 2022}
    \label{fig:DC_data_t}
\end{figure}
\begin{figure}[h]
    \centering
    \includegraphics[width=0.5\textwidth]{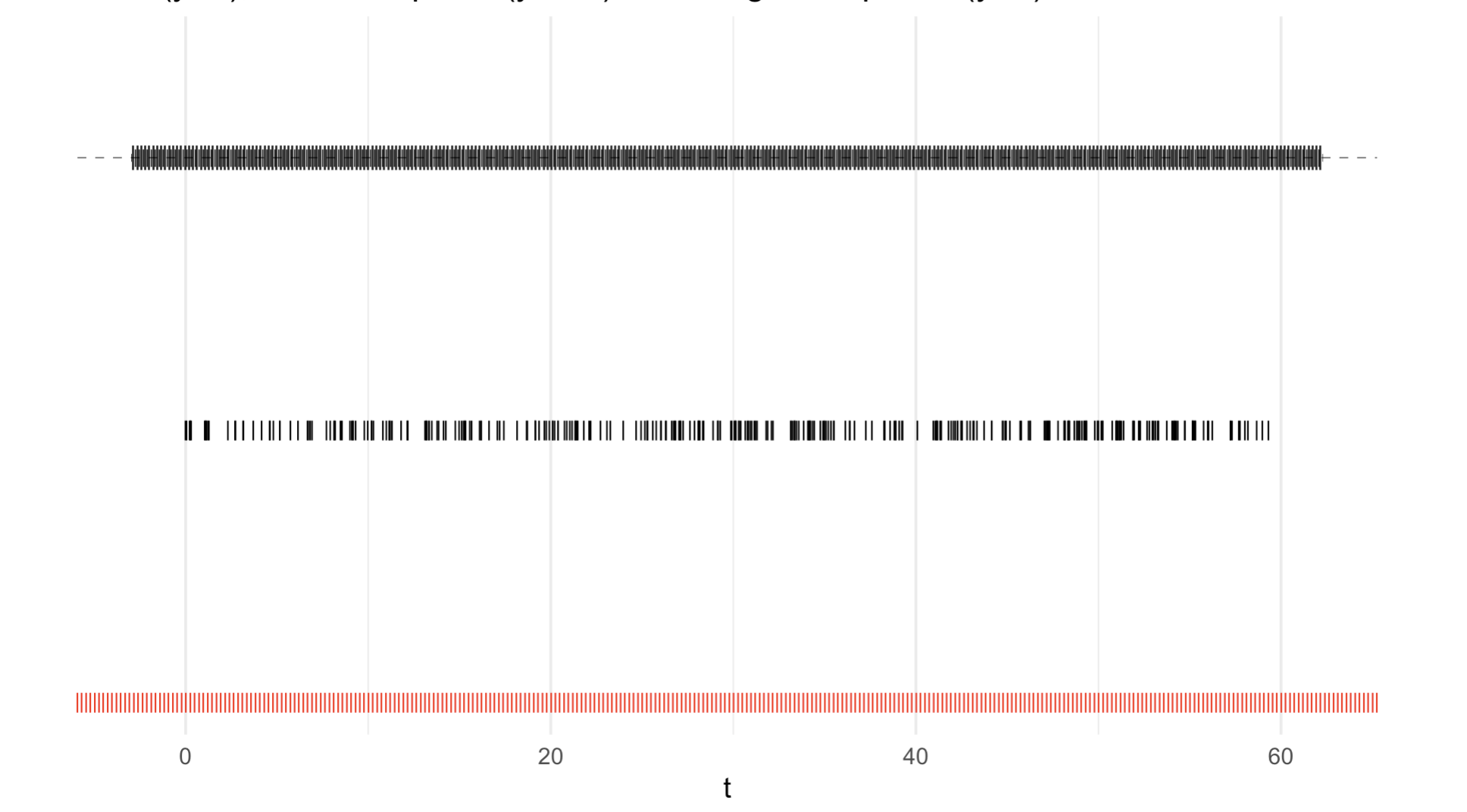}
    \caption{1D time mesh, integration points, and events}
    \label{fig:DC_t_mesh}
\end{figure}

To obtain a smoother and more stable temporal intensity, we place PC priors on the Matérn field defined on the 1D mesh, favouring longer correlation ranges and smaller marginal variability on the linear-predictor scale. Concretely, let \(L=t_{\max}-t_{\min}\) denote the observed span, we set
\begin{equation}
\Pr\{\rho < \tfrac{1}{2}L\}=0.5,
\qquad
\Pr\{\sigma > 0.5\}=0.01,
\end{equation}
where \(\rho\) is the Matérn range and \(\sigma\) is the marginal standard deviation of the latent field on the linear-predictor scale. This choice discourages spurious short-scale oscillations and local noise while retaining sensitivity to sustained structural changes.

We fit the LGCP in \texttt{inlabru} with a linear predictor comprising an intercept and a Matérn field, and perform inference via INLA using the empirical Bayes strategy to balance accuracy and cost. The fit returns posterior summaries for the latent predictor \(\eta(t)\) and, via marginal transformation, for the intensity \(\lambda(t)=\exp\{\eta(t)\}\) (e.g., using \texttt{inla.tmarginal}), as well as marginal posteriors for the hyperparameters, shown in Figure~\ref{fig:spde_hyperposteriors_row}.

\begin{figure}[h]
    \centering
    \includegraphics[width=0.8\textwidth]{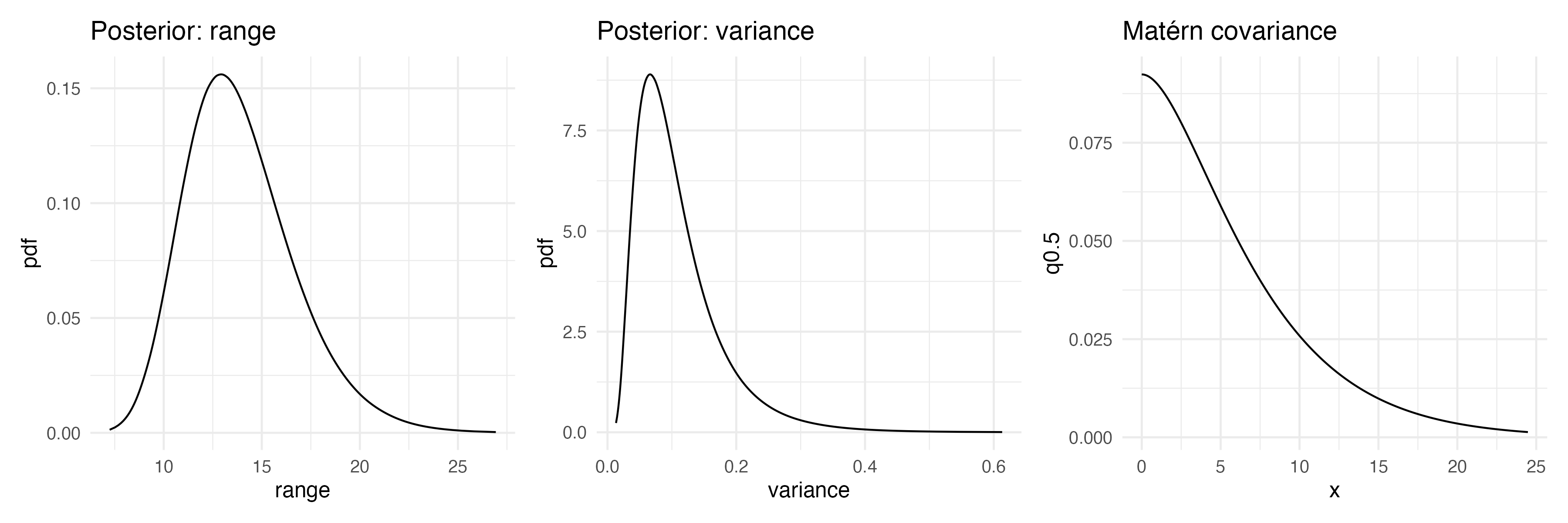}
    \caption{Posteriors for the hyperparameters}
    \label{fig:spde_hyperposteriors_row}
\end{figure}

The range posterior is unimodal and peaks around 14--16 days, with most mass between roughly 10 and 20 days. This rules out very short scales and points to a smooth latent field at the two–week level. The variance posterior is right–skewed, with a mode near 0.1 and most mass well below 0.25, suggesting moderate amplitude rather than sharp spikes. The Matérn covariance curve decays smoothly: correlation is still noticeable up to about 8--10 days and is close to zero by 25 days. The posterior supports a smooth latent field with moderate variance, which matches the two–month crime data and our PC–prior preference for longer ranges and smaller variance. This level of smoothness reduces overfitting to isolated events while keeping sensitivity to sustained changes.

\begin{figure}[h]
    \centering
    \includegraphics[width=0.7\textwidth]{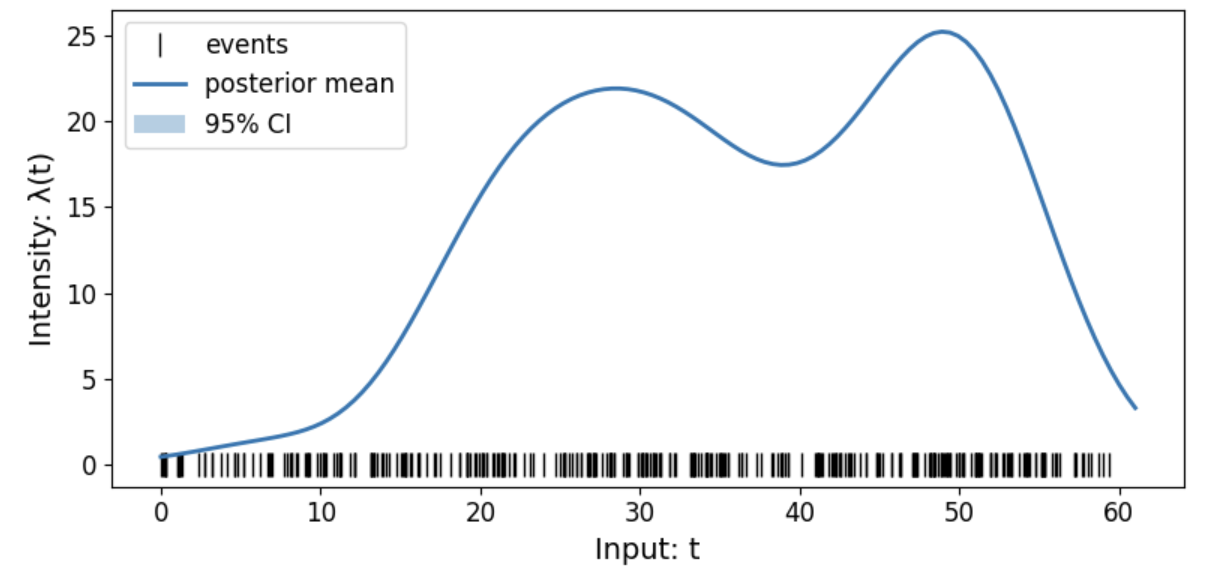}
    \caption{Temporal intensity $\lambda(t)$ (baseline)}
    \label{fig:dc_time_lambda_baseline}
\end{figure}

\section{Bayesian Optimization over 1D Synthetic Intensity}
\label{app:bo_t1d}

We use the Bayesian optimization procedure described in Section~\ref{sec:bo-over-intensity} on the one-dimensional synthetic point-process dataset. The ground-truth intensity is known on the evaluation grid, allowing the region-based simple regret in \eqref{eq:simple-regret} to be computed directly.

\paragraph{LGCP specification.}
The temporal domain is
\(
\mathcal T=[t_{\min},t_{\max}]
\)
with length
\(
L=t_{\max}-t_{\min}.
\)
We construct a one-dimensional SPDE domain mesh with \(0.3L\) padding on both sides of \(\mathcal T\), and a separate integration mesh with \(0.1L\) padding. Both meshes contain 200 equally spaced nodes. The LGCP consists of an intercept and a one-dimensional Mat\'ern SPDE field. Numerical integration is implemented using \texttt{fmesher::fm\_int}, and all posterior predictions used by BO are evaluated on the common candidate grid \(\mathcal G\). Inference is performed with \texttt{inlabru::lgcp} using the empirical Bayes integration strategy.

For ISBO, the Mat\'ern field uses the PC-prior specification
\begin{align}
    \Pr\{\rho<L/10\}
    &=
    0.3,
    &
    \Pr\{\sigma>0.2\}
    &=
    0.5.
    \label{eq:bo1d-pc-prior}
\end{align}
The \emph{Standard prior} ablation replaces this with \texttt{inla.spde2.matern}, using nominal range \(L/10\), nominal variance \(0.2^2\), and \texttt{theta.prior.prec}\(=1\). All other components are kept unchanged.

\paragraph{BO settings.}
The acquisition function follows
\eqref{eq:isbo-acquisition} and the decreasing schedule in
\eqref{eq:kappa-schedule}. We use
\begin{align}
    \kappa_{\mathrm{start}}
    &=
    9,
    &
    \kappa_{\mathrm{final}}
    &=
    4,
    &
    T
    &=
    15.
\end{align}
The \emph{Standard UCB} ablation instead fixes
\(
\kappa_t=6
\)
for all iterations. Acquisition is performed directly on the latent
\(\eta\)-scale. 

The query radius in
\eqref{eq:bo-query-window} is set to
\begin{align}
    r=1.
\end{align}
Previously revealed regions and previous query centres are excluded from subsequent UCB maximization according to
\eqref{eq:revealed-region}--\eqref{eq:bo-next-query}. We do not apply
additional mask dilation beyond the queried region. 

\paragraph{Initialization and repetitions.}
For each experimental seed, three distinct initial centres are sampled from the evaluation grid, excluding
\begin{align}
    [45,55],
\end{align}
so that initialization does not directly reveal the dominant peak region. Within each seed, the same three initial centres are used for all compared methods. We run ten seeds,
\begin{align}
    13,\,23,\,33,\,43,\,53,\,63,\,73,\,83,\,93,\,103.
\end{align}
The posterior obtained from these shared initial regions is recorded as iteration \(0\), providing a common starting point for the regret curves.

\paragraph{Ablation study.}
The four configurations are
\begin{align}
    \text{ISBO}
    &:~
    \text{PC prior}
    +
    \text{time-varying UCB},
    \\
    \text{Standard prior}
    &:~
    \text{standard Mat\'ern prior}
    +
    \text{time-varying UCB},
    \\
    \text{Standard UCB}
    &:~
    \text{PC prior}
    +
    \text{fixed UCB},
    \\
    \text{Random}
    &:~
    \text{PC prior}
    +
    \text{random querying}.
\end{align}
For the random policy, query centres are sampled from the same evaluation grid without replacement. Hence, \emph{Standard prior} isolates the effect of the PC prior, while \emph{Standard UCB} isolates the effect of the time-varying exploration schedule.

\paragraph{Stepwise visualization.}
Figure~\ref{fig:bo_synth_steps} provides a qualitative illustration of the sequential BO process of ISBO on the synthetic one-dimensional intensity. At the initial stage, the posterior mean remains diffuse and the acquisition function places substantial weight on unexplored regions. By iteration 5, the main peaks of the underlying intensity have already been identified and the posterior uncertainty is visibly reduced around the revealed regions. By iteration 13, the posterior mean closely tracks the ground-truth intensity and the uncertainty band has contracted substantially, indicating that the surrogate has largely stabilized. This visualization is consistent with the simple-regret results in the main text: ISBO first explores broadly, then progressively concentrates on the high-intensity regions.

\begin{figure*}[t]
  \centering
  \begin{subfigure}[t]{0.31\textwidth}
    \centering
    \includegraphics[width=\linewidth]{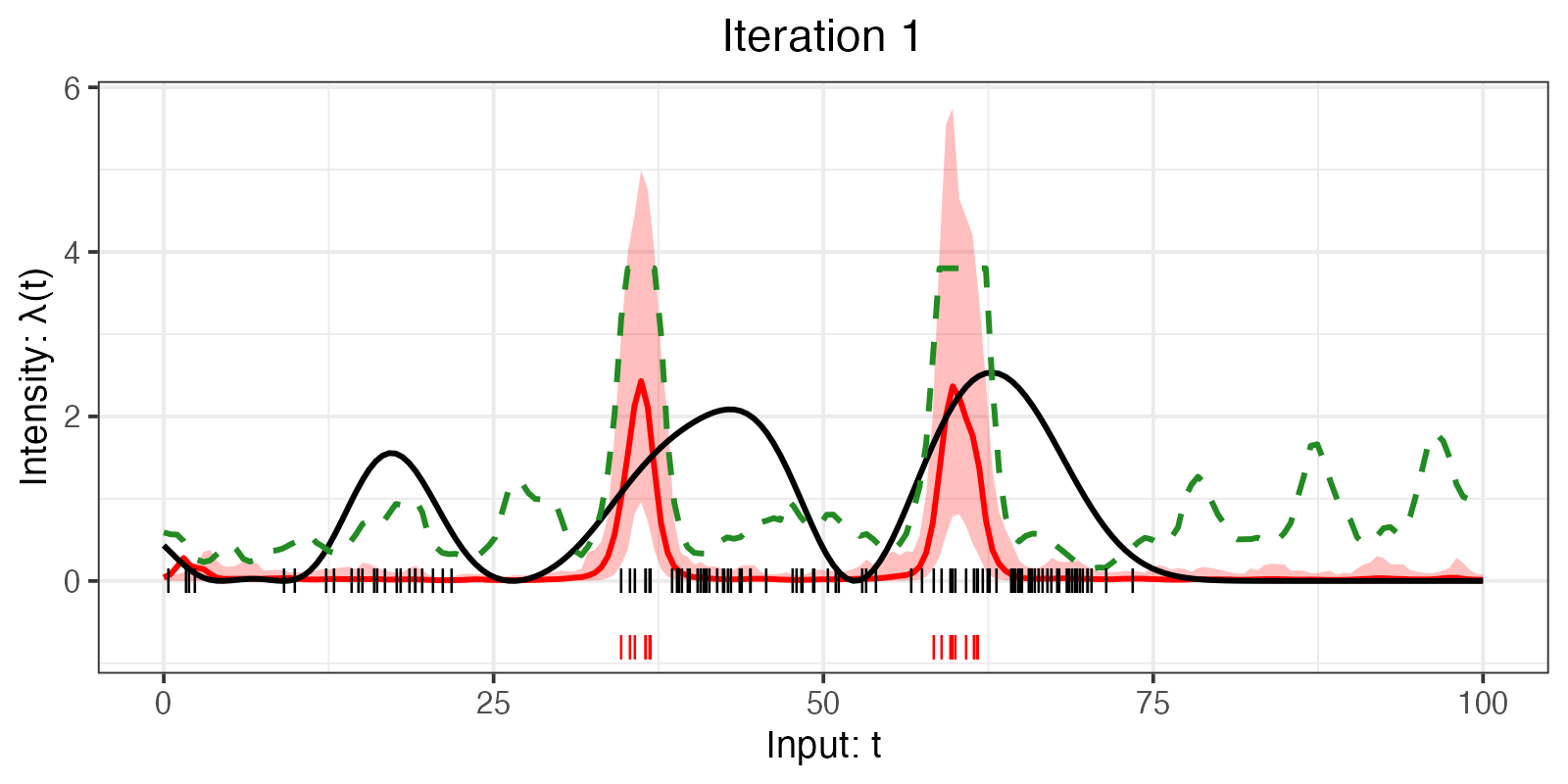}
    \caption{Iteration 1}
  \end{subfigure}
  \hfill
  \begin{subfigure}[t]{0.31\textwidth}
    \centering
    \includegraphics[width=\linewidth]{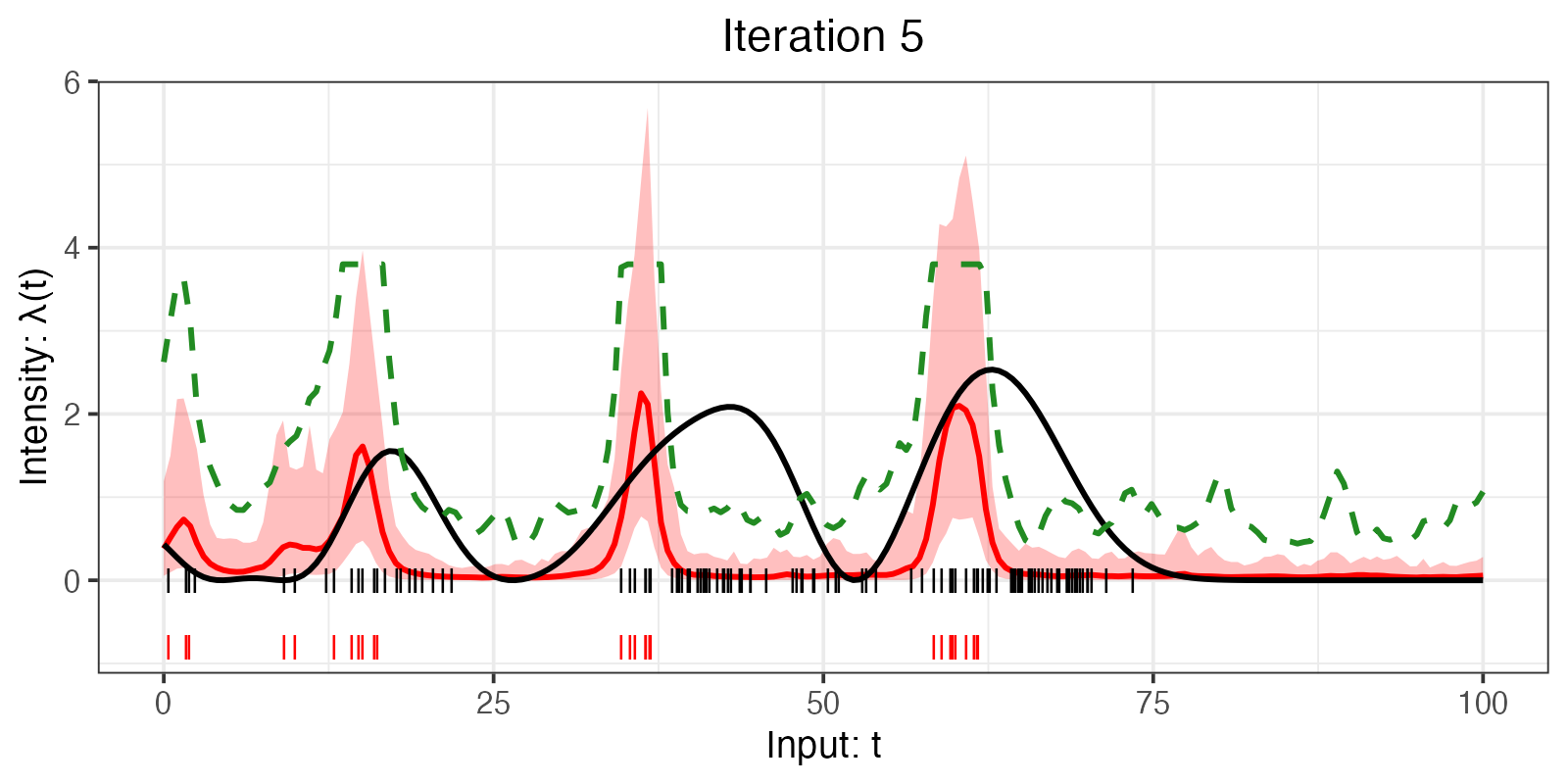}
    \caption{Iteration 5}
  \end{subfigure}
  \hfill
  \begin{subfigure}[t]{0.31\textwidth}
    \centering
    \includegraphics[width=\linewidth]{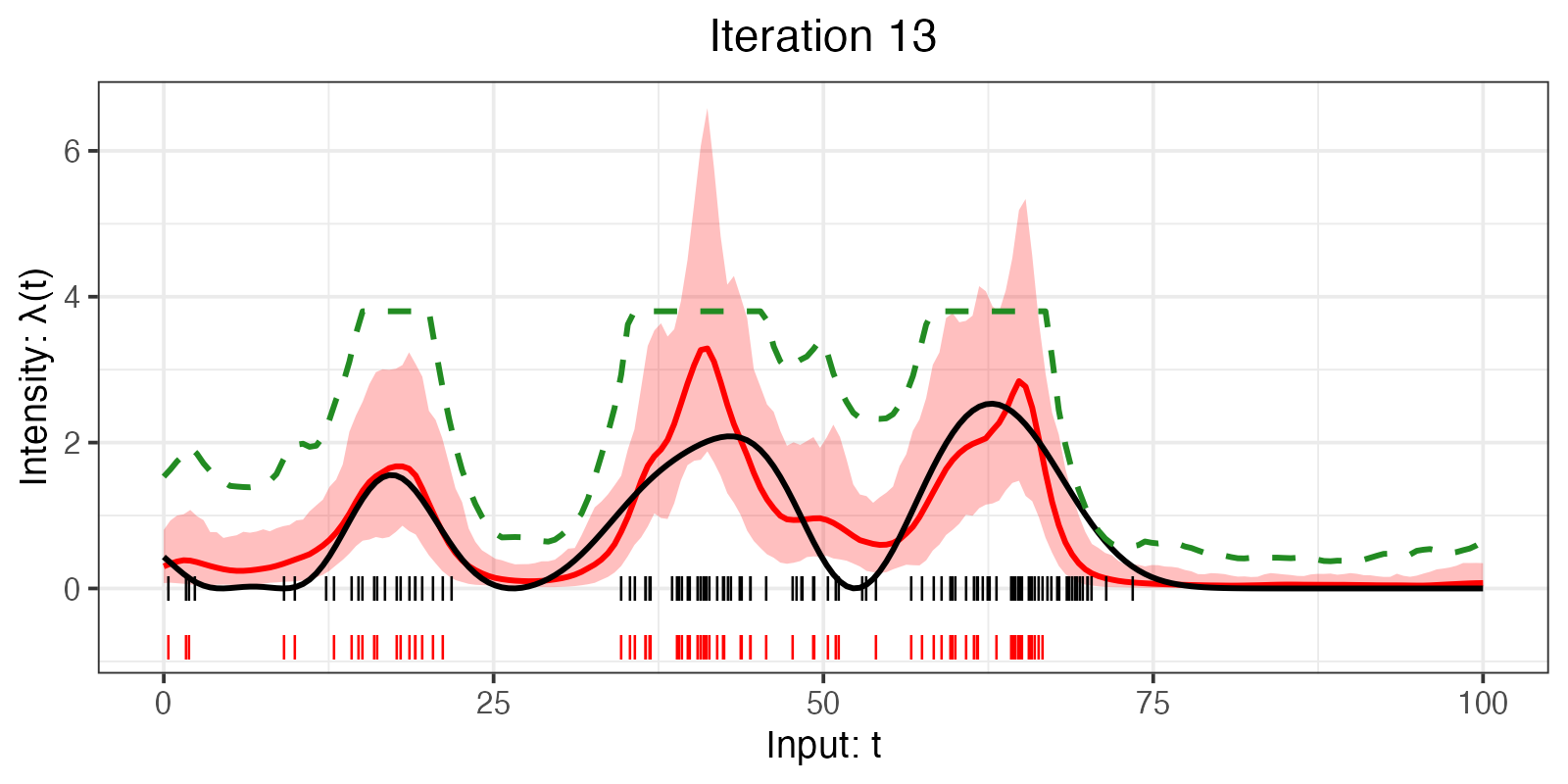}
    \caption{Iteration 13}
  \end{subfigure}

  \caption{Stepwise visualization of ISBO over the synthetic
  intensity. The red line is the posterior mean intensity, the pink band is
  the 95\% credible interval, the green dashed line is the acquisition
  function, the black line is the ground-truth intensity, the black vertical
  ticks denote all incident times, and the red vertical ticks denote the
  revealed incident times.}
  \label{fig:bo_synth_steps}
\end{figure*}

\newpage

\section{Pseudocode for 1D Temporal BO via INLA--SPDE--LGCP}
\label{app:ps_code1}

\begin{table}[h]
\centering
\renewcommand{\arraystretch}{1.10}
\begin{tabular}{p{0.08\linewidth}p{0.84\linewidth}}
\hline
\multicolumn{2}{l}{
\textbf{Algorithm 2 \quad 1D Temporal ISBO Implementation}
}
\\
\hline

\textbf{Input}
&
Event times $\mathcal A=\{t_i\}$; candidate grid $\mathcal G$;
initial centres $\mathcal C_0$; query radius $r$; budget $T$;
$\kappa_{\mathrm{start}},\kappa_{\mathrm{final}}$.
\\

\textbf{Initialise}
&
Construct the 1D SPDE domain and integration meshes, define the Mat\'ern
LGCP surrogate, reveal events around $\mathcal C_0$, and fit the initial
posterior.
\\

\hline
\multicolumn{2}{l}{\textbf{for} $t=1,\ldots,T$ \textbf{do}}
\\

1 &
Predict $\mu_{\eta,t-1}(x)$ and $\sigma_{\eta,t-1}(x)$ on $\mathcal G$.
\\

2 &
Compute
\[
\kappa_t
=
\kappa_{\mathrm{start}}
-
\frac{t-1}{T-1}
(\kappa_{\mathrm{start}}-\kappa_{\mathrm{final}})
\]
and
\[
a_t(x)
=
\mu_{\eta,t-1}(x)+\kappa_t\sigma_{\eta,t-1}(x).
\]
\\

3 &
Mask previously revealed regions and select
\[
x_t
\in
\arg\max_{x\in\mathcal C_{t-1}} a_t(x).
\]
\\

4 &
Reveal events in
$\mathcal W(x_t;r)$, update the masked candidate set, and refit the LGCP.
\\

5 &
Record posterior summaries and region-based optimization metrics.
\\

\multicolumn{2}{l}{\textbf{end for}}
\\
\hline

\textbf{Return}
&
Selected centres, revealed events, and per-iteration LGCP posteriors and
optimization metrics.
\\

\hline
\end{tabular}

\caption{\small Detailed one-dimensional implementation of ISBO.}
\label{alg:bo-inla-lgcp-1d}
\end{table}

\section{2D Spatial Intensity Estimation (Porto Taxi Trajectories)}
\label{app:2DSint}

\paragraph{Data and spatial model.}
Let the spatial observation window be the unit square \(S=[0,1]\times[0,1]\). We linearly rescale the raw taxi starting locations to \(S\) and restrict both inference and prediction to this domain. The full dataset contains 106{,}732 trips, as shown in Figure~\ref{fig:taxi_full_data}. For visualization, we fit the model using subsets of 200 and 66{,}666 events.

\begin{figure}[h]
    \centering
    \includegraphics[width=0.35\textwidth]
    {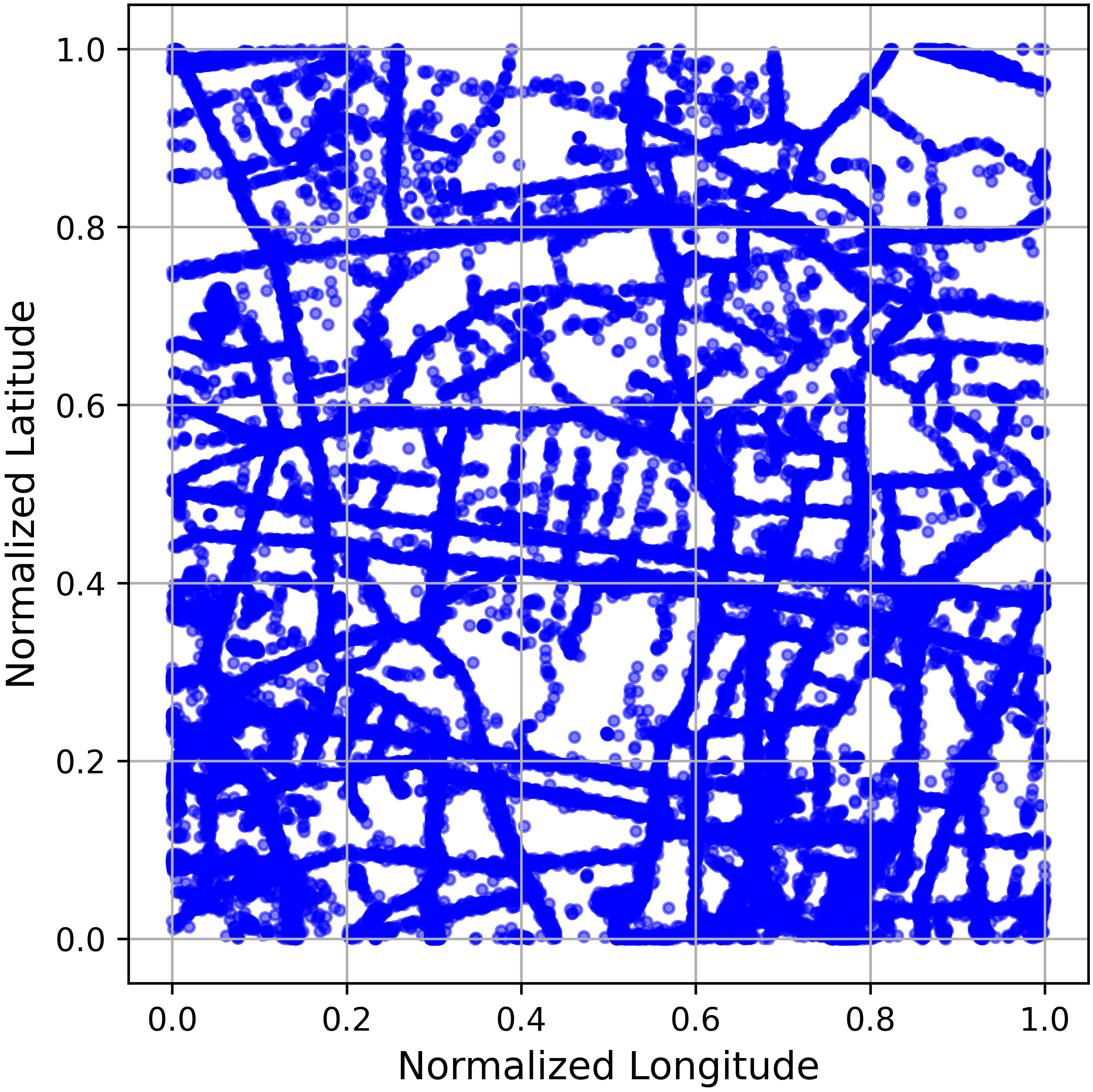}
    \caption{Starting locations of the full Porto taxi dataset.}
    \label{fig:taxi_full_data}
\end{figure}

We construct a triangular mesh \(\mathcal M_s\) over \(S\) using \texttt{inla.mesh.2d} with \texttt{max.edge = c(0.1,0.2)} and \texttt{cutoff = 0.01}. The latent spatial field follows a Mat\'ern--SPDE model with PC priors
\begin{align}
    \Pr\{\rho<0.2\}
    &=
    0.5,
    &
    \Pr\{\sigma>1\}
    &=
    0.01.
\end{align}
The LGCP predictor consists of an intercept and the spatial Mat\'ern field. Inference is performed using \texttt{inlabru::lgcp} with empirical Bayes integration.

For comparison, we additionally fit the Nystr\"om-based method of \citet{meiBayesianOptGCPMstD} to the same 200-event subset. Its estimated intensity is shown in Figure~\ref{fig:taxi_baseline_inten}.

\begin{figure}[h]
    \centering
    \includegraphics[width=0.35\textwidth]
    {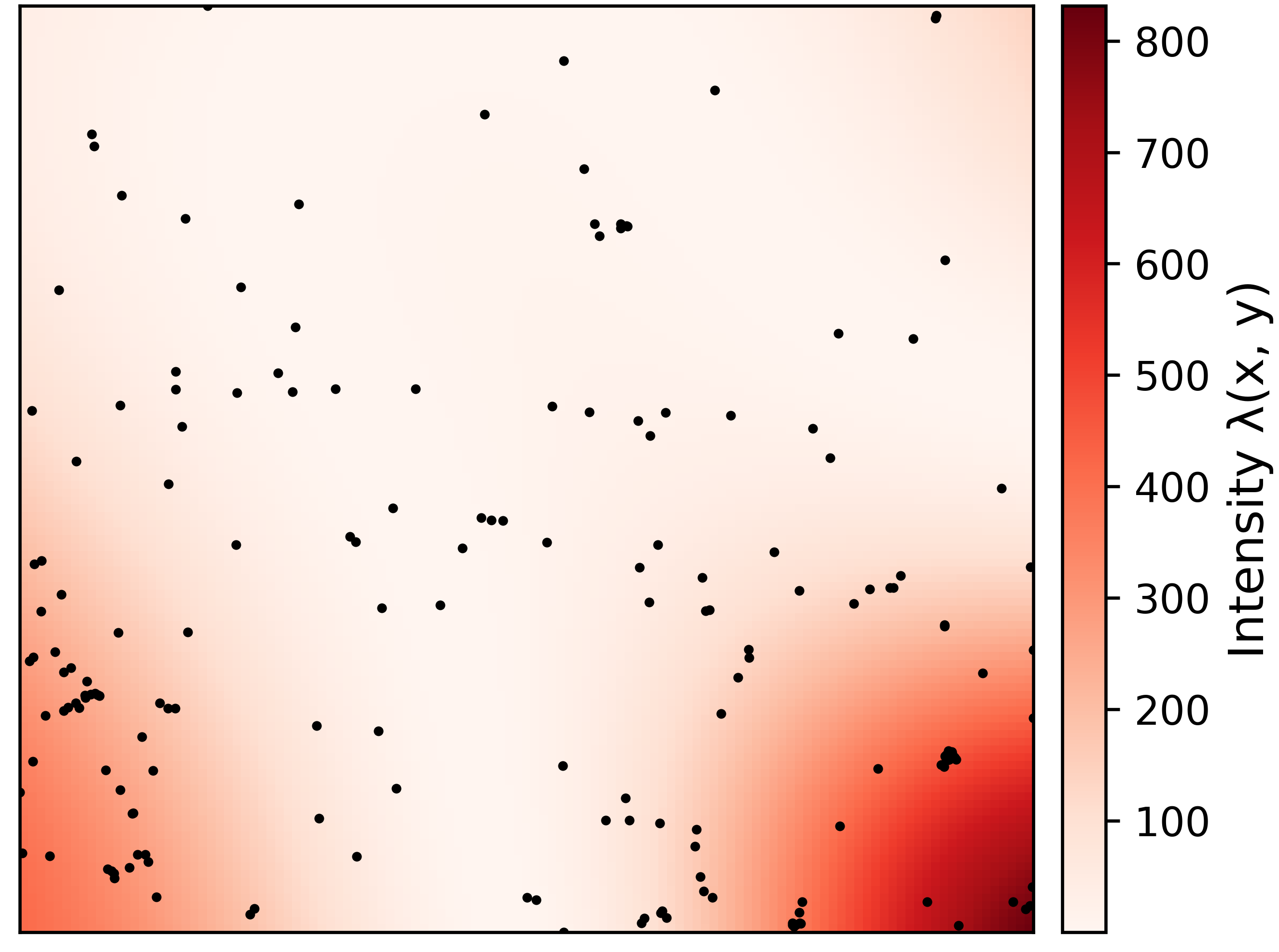}
    \caption{Spatial intensity estimated by the Nystr\"om baseline using
    200 taxi trips. Black points denote the observed starting locations.}
    \label{fig:taxi_baseline_inten}
\end{figure}

\paragraph{Posterior prediction.}
Predictions are evaluated on a pixel grid \(\mathcal P\) generated from \(\mathcal M_s\) and masked to \(S\). We report both the latent predictor \(\eta(\mathbf s)\) and the response-scale intensity \(\lambda(\mathbf s)=\exp\{\eta(\mathbf s)\}\). Figure~\ref{fig:2d_intensity_200_6666} compares the posterior mean surfaces obtained from 200 and 66{,}666 events.

\begin{figure}[h]
  \centering
  \includegraphics[width=0.5\linewidth]
  {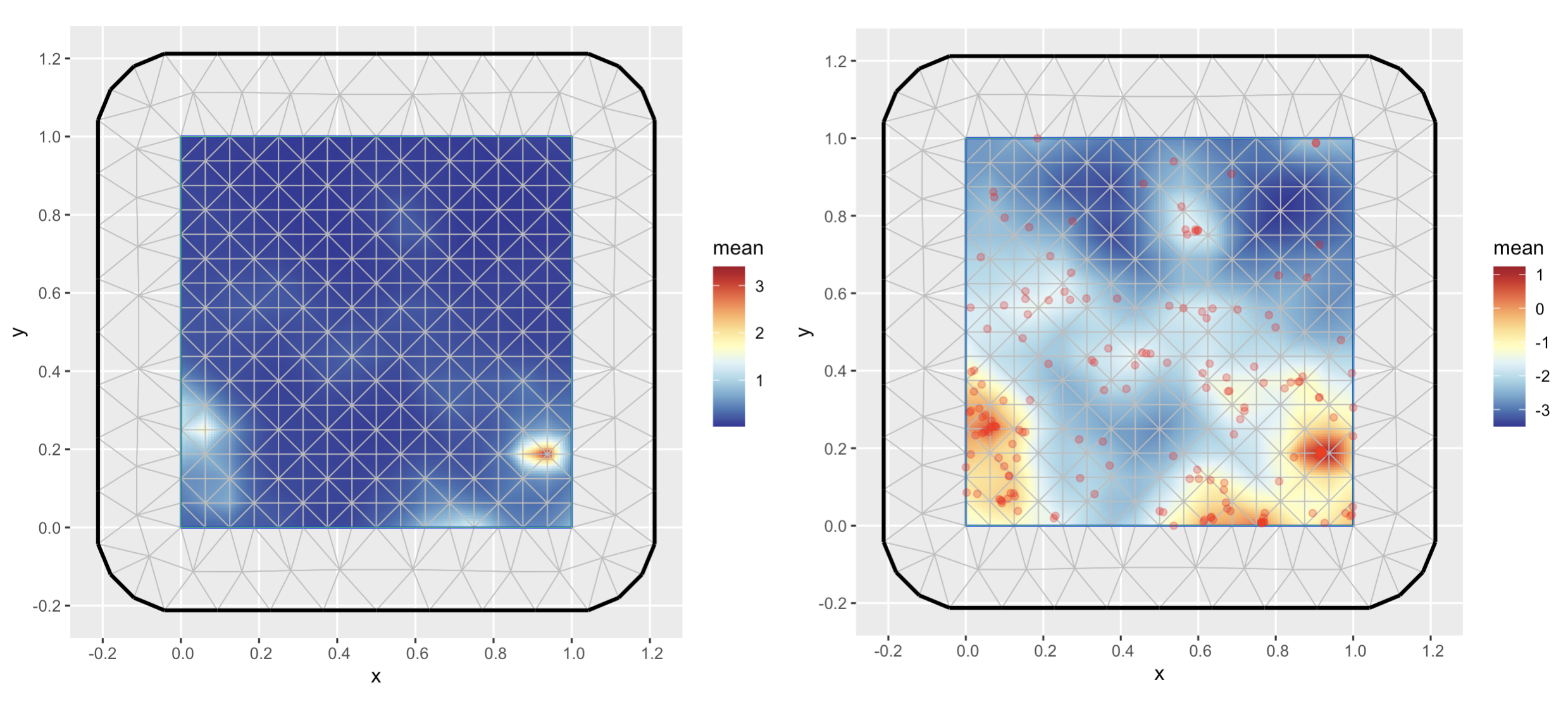}\\[0.8ex]
  \includegraphics[width=0.5\linewidth]
  {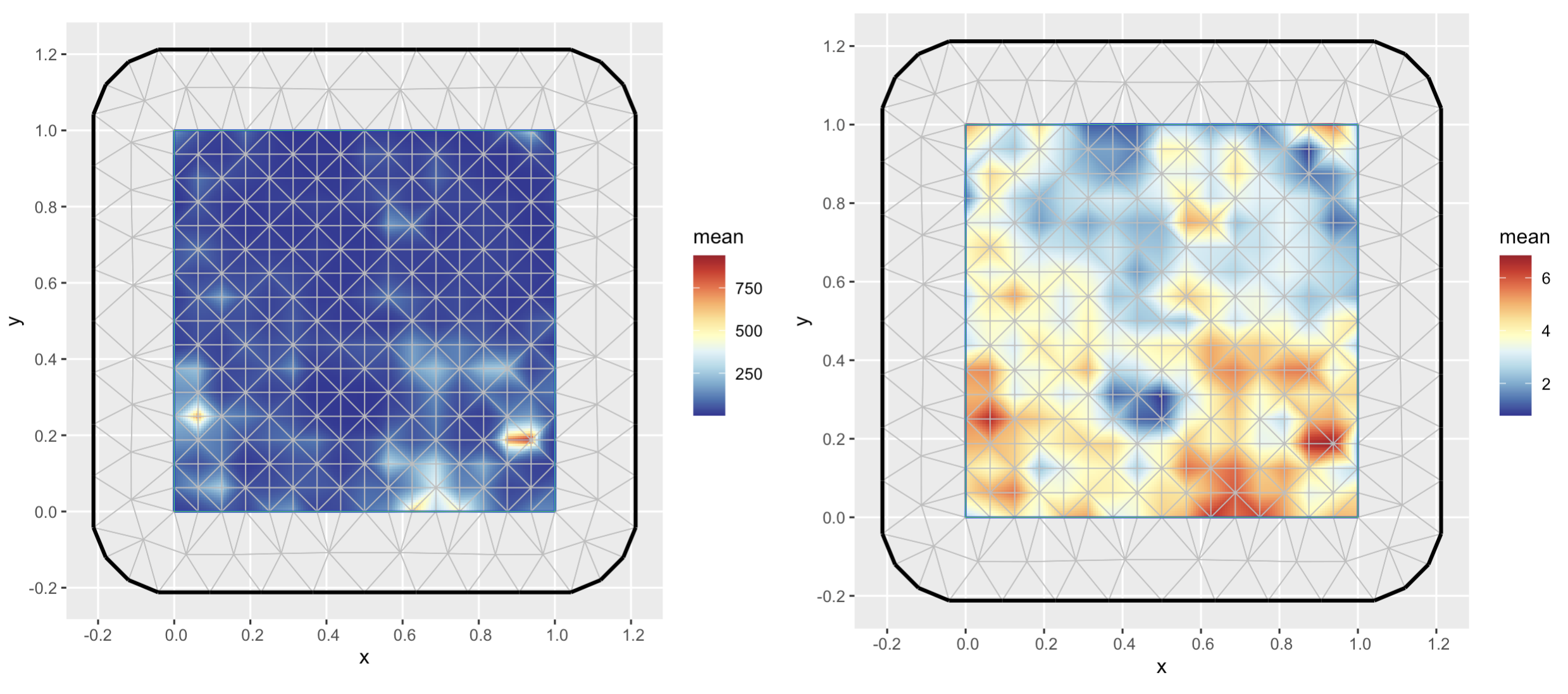}
  \caption{
  Posterior spatial intensity for the Porto taxi data using
  200 events (top) and 66{,}666 events (bottom).
  Left panels show the response scale and right panels the latent scale.
  }
  \label{fig:2d_intensity_200_6666}
\end{figure}

\paragraph{Runtime benchmark.}
All runtime experiments are performed locally on the same MacBook Air M2 with 16\,GB RAM, without HPC resources. We compare INLA--SPDE with the Nystr\"om inference procedure of \citet{meiBayesianOptGCPMstD} and the Laplace and variational approximations implemented in SCAMPR \citep{dovers2023fast}.

For the small-to-moderate benchmark, we use common event subsets ranging from 200 to 2{,}500 events and compare all four inference procedures. For the large-scale benchmark, we compare INLA--SPDE, SCAMPR--Laplace, and SCAMPR--VA using subsets from 5{,}000 to 55{,}000 events. The same sampled events are supplied to the competing methods at each sample size, and the reported runtime is the mean fitting time across repeated runs. The resulting benchmarks are reported in Figure~\ref{fig:runtime_benchmark}.

\section{Bayesian Optimization over 2D Spatial Intensity}
\label{app:bo_spatial}

We perform Bayesian optimization directly over the spatial domain of the Washington, D.C.\ firearm-related crime data. Event coordinates are projected to UTM zone 18N (EPSG:26918), and only events inside the D.C.\ boundary are retained. 

\paragraph{Spatial LGCP setup.}
Let \(S\) denote the spatial domain and let \(\ell_S\) be the shorter side of its bounding box. We construct the triangular SPDE mesh using
\begin{align}
    \texttt{max.edge}
    &=
    \left(
        \frac{\ell_S}{25},
        \frac{\ell_S}{8}
    \right),
    &
    \texttt{cutoff}
    &=
    \frac{\ell_S}{150}.
\end{align}
The Mat\'ern field uses PC priors
\begin{align}
    \Pr\{\rho<\ell_S/5\}
    &=
    0.01,
    &
    \Pr\{\sigma>1\}
    &=
    0.01.
\end{align}
The LGCP predictor consists of an intercept and the spatial SPDE field, and inference is performed using \texttt{inlabru::lgcp} with the empirical Bayes strategy. 
\paragraph{Spatial BO settings.}
The BO candidate locations are the spatial pixels generated from the same mesh and masked to \(S\). For each seed, two initial centres are sampled from this candidate set. A query at \(\mathbf s_t\) reveals all events within a disk of radius
\begin{align}
    r=1000\ \mathrm{m}.
\end{align}
Previously queried regions are masked with no additional dilation. We use \(T=25\) BO iterations and the decreasing exploration schedule in
\eqref{eq:kappa-schedule}, with
\begin{align}
    \kappa_{\mathrm{start}}=10,
    \qquad
    \kappa_{\mathrm{final}}=2.
\end{align}
The experiment is repeated over ten seeds,
\begin{align}
    13,\,23,\,33,\,43,\,53,\,63,\,73,\,83,\,93,\,103.
\end{align}

At each iteration, the spatial LGCP is refitted using only the currently revealed events, and posterior predictions are evaluated on the common spatial candidate grid. The next query location is selected after masking previously queried disks. 

\begin{figure*}[t]
    \centering
    \begin{minipage}[t]{0.49\textwidth}
        \centering
        \includegraphics[width=\linewidth]
        {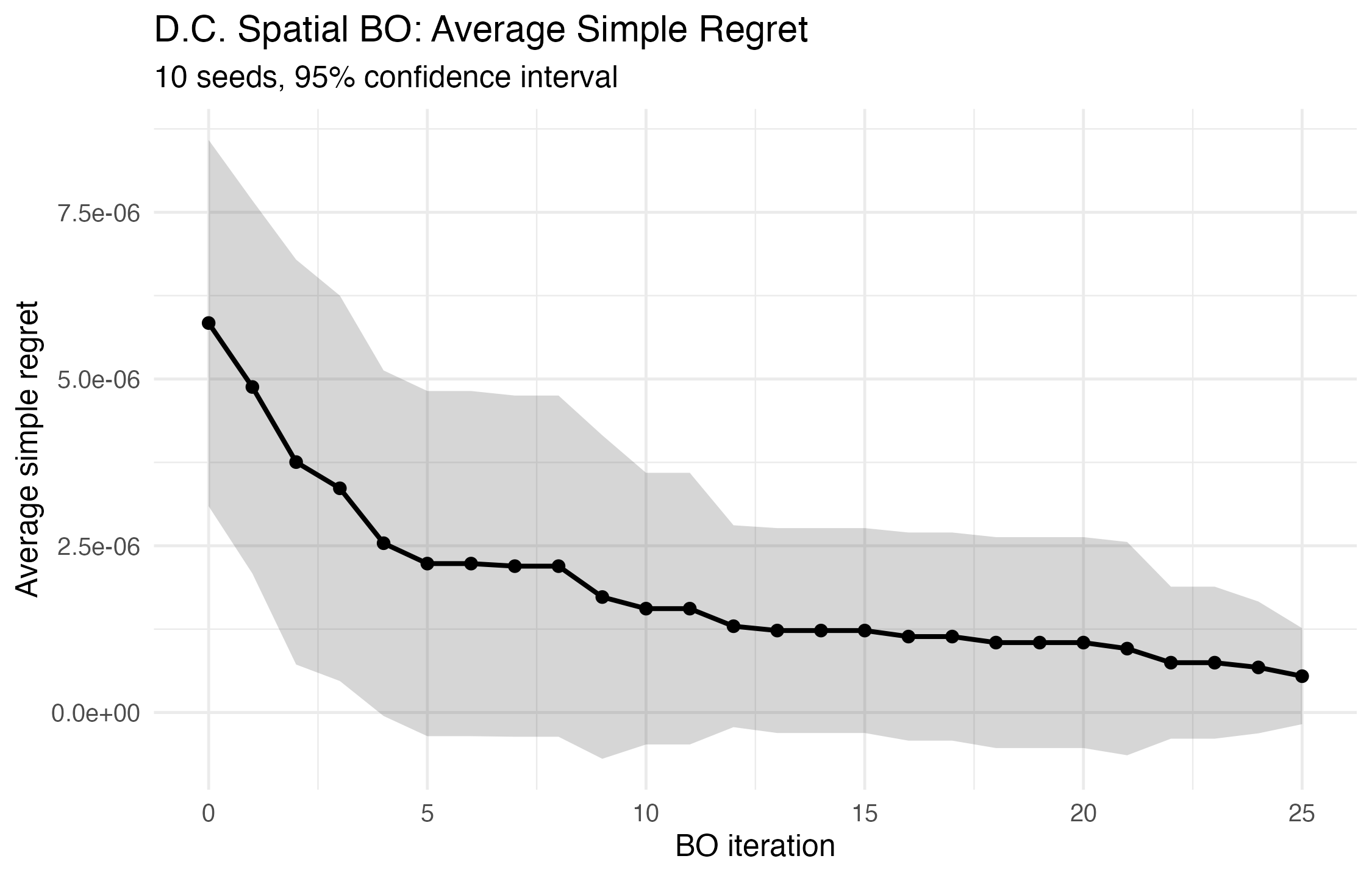}
        \small (a) Average simple regret with 95\% confidence intervals.
    \end{minipage}
    \hfill
    \begin{minipage}[t]{0.49\textwidth}
        \centering
        \includegraphics[width=\linewidth]
        {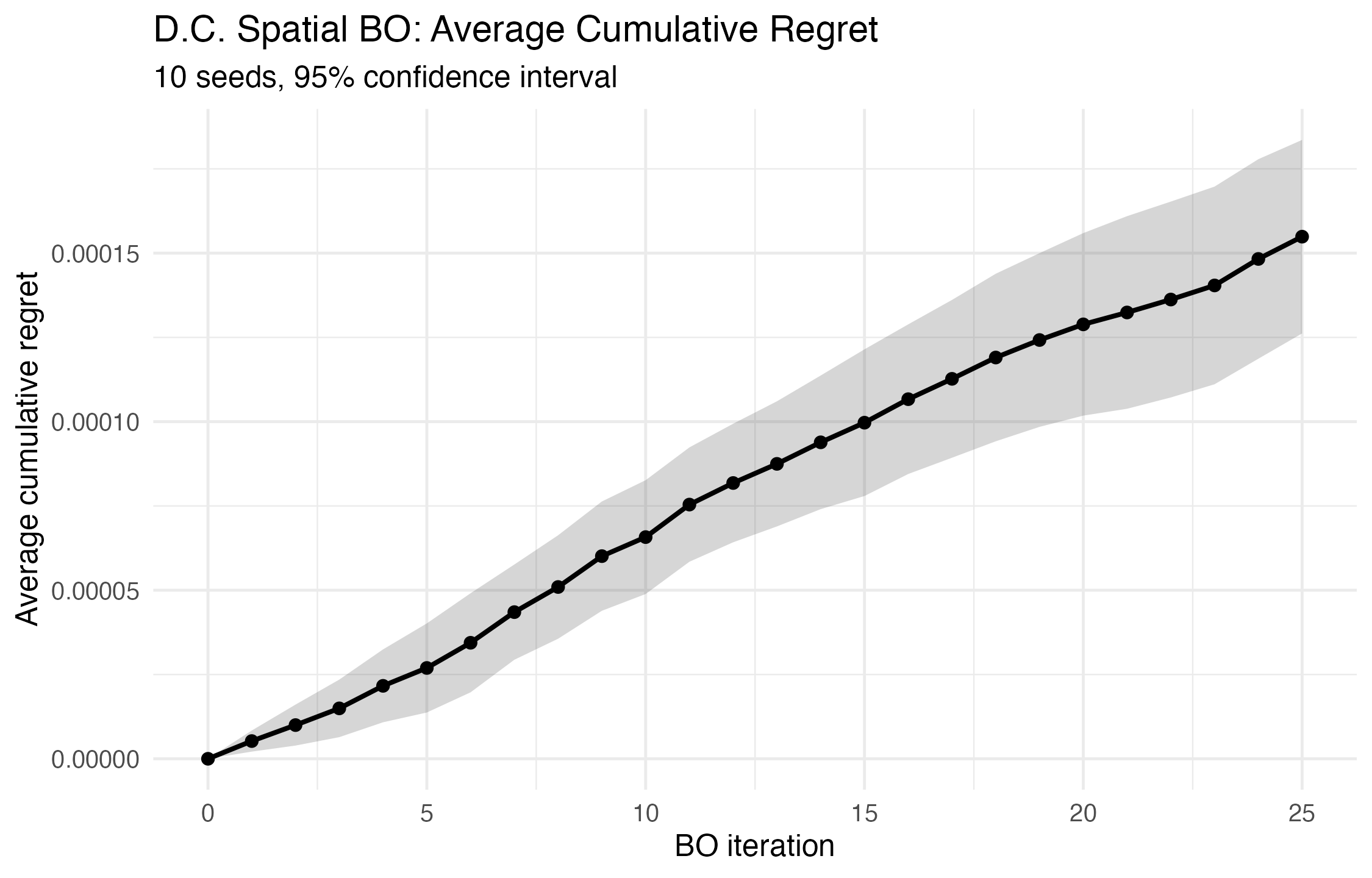}
        \small (b) Average cumulative regret with 95\% confidence intervals.
    \end{minipage}

    \caption{
        Direct two-dimensional spatial Bayesian optimization on the
        Washington, D.C.\ crime data over 10 seeds. Shaded regions denote
        95\% confidence intervals.
    }
    \label{fig:dc_spatial_regret}
\end{figure*}

\paragraph{Reference intensity and evaluation.}
Since the true spatial intensity is unavailable for the real data, we use the LGCP intensity fitted to the full dataset as the reference surface, \(\widehat{\lambda}_{\mathrm{ref}}(\mathbf s)\), with
\begin{align}
    \widehat{\lambda}^{\star}_{s}
    =
    \max_{\mathbf s\in S}
    \widehat{\lambda}_{\mathrm{ref}}(\mathbf s).
\end{align}
For each queried disk, its value is defined as the maximum reference intensity contained within that region. Simple and cumulative regret are then computed using the region-based definitions in \eqref{eq:instantaneous-regret}--\eqref{eq:simple-regret}, replacing the unknown ground-truth intensity with \(\widehat{\lambda}_{\mathrm{ref}}\).

The reported regret curves are averaged over the ten seeds, with 95\% Student-\(t\) confidence intervals for the mean. Figure~\ref{fig:dc_spatial_regret} summarizes the spatial BO performance over 10 seeds. The average simple regret decreases steadily throughout the 25 iterations, showing that ISBO progressively identifies spatial regions closer to the maximum of the full-data reference intensity. The confidence interval also contracts in later iterations, indicating increasingly consistent peak localization across runs. As expected, cumulative regret increases with the number of queries because it aggregates the regret incurred at every iteration.

\paragraph{Visualization.}
Figure~\ref{fig:bo_spDC_steps} provides a qualitative visualization of the spatial BO process. At the initial stage, the posterior intensity surface is still diffuse because only a small portion of the crime events has been revealed. By around iteration 11, the main hotspot structure has already emerged, and the posterior surface becomes much more informative. By iteration 25, after more regions have been queried, the spatial intensity map is substantially sharper and the dominant high-intensity areas are clearly identified. This qualitative behaviour is consistent with the regret curves: ISBO progressively refines the spatial surrogate and concentrates exploration on the most informative regions.

\begin{figure*}[t]
  \centering
  \begin{minipage}[t]{0.33\textwidth}
    \includegraphics[width=\linewidth]{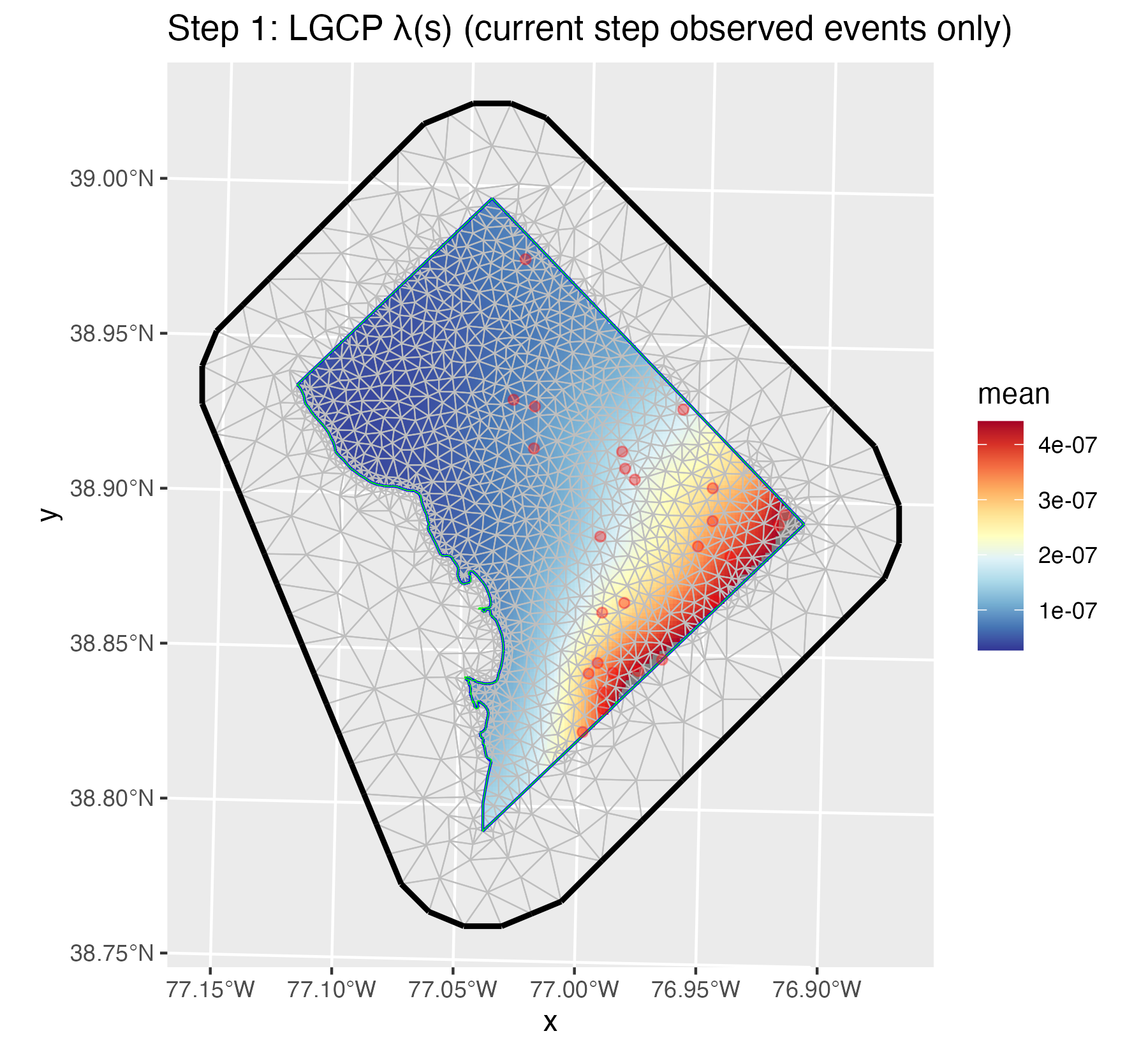}
  \end{minipage}\hfill
  \begin{minipage}[t]{0.33\textwidth}
    \includegraphics[width=\linewidth]{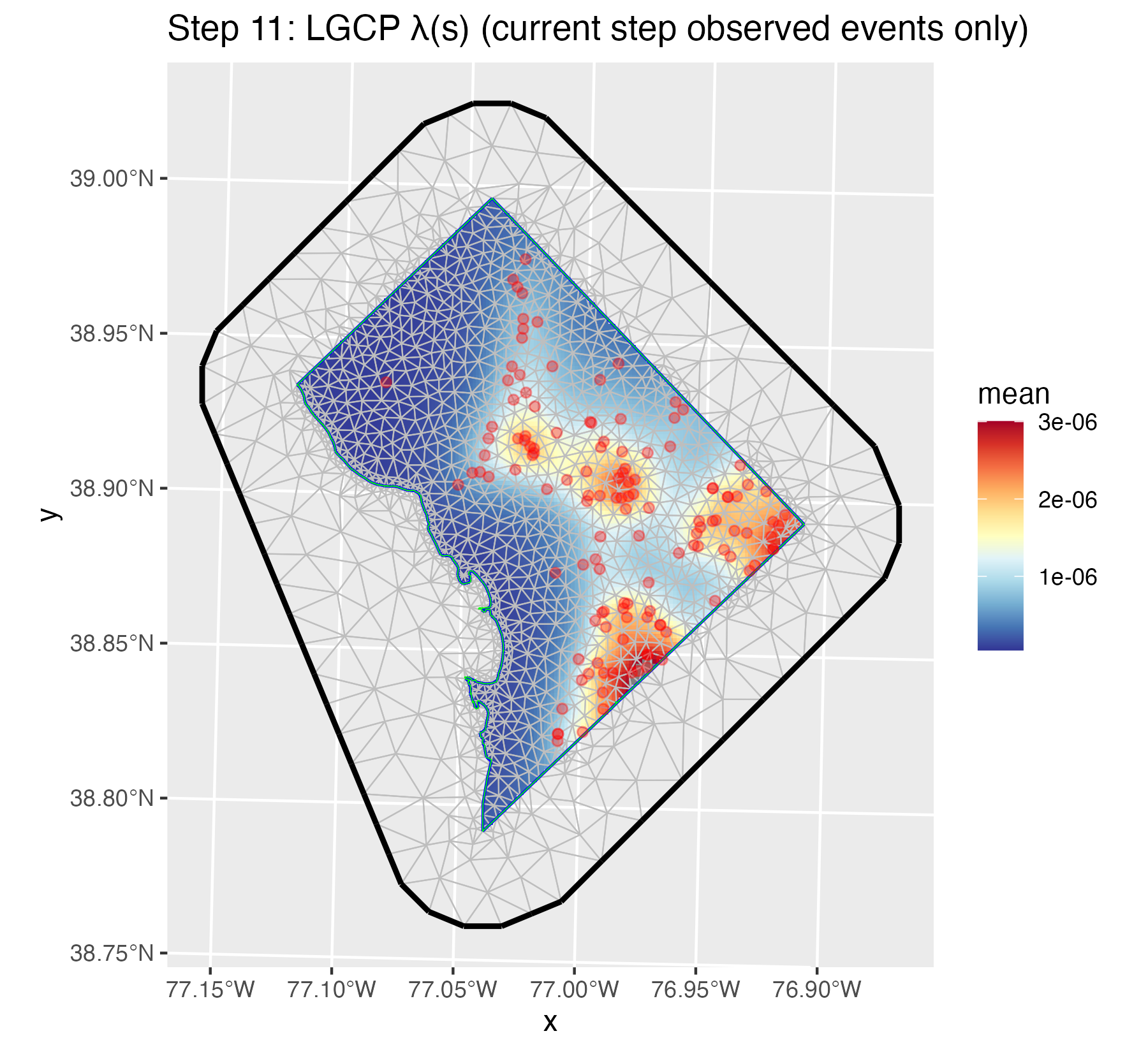}
  \end{minipage}\hfill
  \begin{minipage}[t]{0.33\textwidth}
    \includegraphics[width=\linewidth]{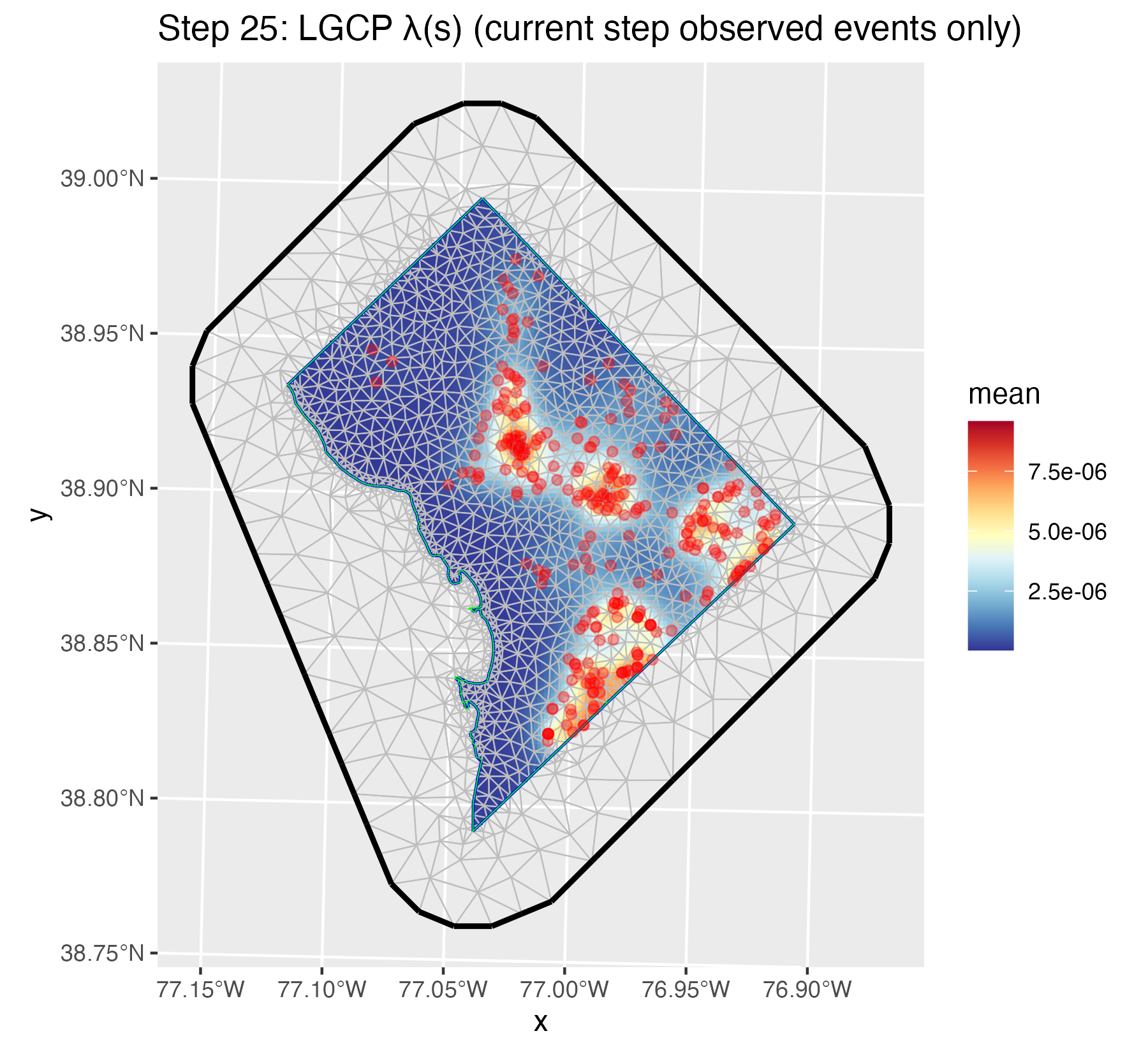}
  \end{minipage}
  \caption{Visualization of BO step-wise spatial intensity, red dots: used crime locations for current step}
  \label{fig:bo_spDC_steps}
\end{figure*}

\newpage

\section{Pseudocode for 2D Spatial BO via INLA--SPDE--LGCP}
\label{app:ps_code2}

\begin{table}[h]
\centering
\renewcommand{\arraystretch}{1.10}
\begin{tabular}{p{0.08\linewidth}p{0.84\linewidth}}
\hline
\multicolumn{2}{l}{
\textbf{Algorithm 3 \quad 2D Spatial ISBO Implementation}
}
\\
\hline

\textbf{Input}
&
Spatial events $\mathcal A=\{\mathbf s_i\}$; spatial domain $S$; candidate grid $\mathcal G_s$; initial centres $\mathcal C_0$; query radius $r$; budget $T$; $\kappa_{\mathrm{start}},\kappa_{\mathrm{final}}$.
\\

\textbf{Initialise}
&
Project event locations to a metric coordinate system and retain events inside $S$. Construct the triangular SPDE mesh, define the Mat\'ern LGCP surrogate, and generate the common spatial candidate grid $\mathcal G_s$ from the prediction pixels.
\\

&
Sample the initial centres $\mathcal C_0$, reveal all events within radius $r$ of these centres, and fit the initial spatial LGCP posterior. Use the full-data fitted intensity $\widehat{\lambda}_{\mathrm{ref}}(\mathbf s)$ as the reference surface for regret evaluation.
\\

\hline
\multicolumn{2}{l}{\textbf{for} $t=1,\ldots,T$ \textbf{do}}
\\

1 &
\textbf{Fit and predict:} fit the spatial LGCP to the currently revealed events and obtain posterior mean and uncertainty over $\mathcal G_s$.
\\

2 &
\textbf{Update exploration:}
compute
\[
\kappa_t
=
\kappa_{\mathrm{start}}
-
\frac{t-1}{T-1}
(\kappa_{\mathrm{start}}-\kappa_{\mathrm{final}})
\]
and construct the spatial UCB acquisition.
\\

3 &
\textbf{Mask and select:}
exclude candidate locations within radius $r$ of previously queried centres,
then select the admissible location with the largest acquisition value.
\\

4 &
\textbf{Reveal region:}
define the queried disk
\[
\mathcal W(\mathbf s_t;r)
=
\{\mathbf s\in S:
\|\mathbf s-\mathbf s_t\|\le r\},
\]
and reveal all observed events contained in this region.
\\

5 &
\textbf{Update optimization metrics:}
evaluate the maximum reference intensity inside
$\mathcal W(\mathbf s_t;r)$ and update the region-based instantaneous,
simple, and cumulative regret.
\\

6 &
\textbf{Update surrogate:}
add $\mathbf s_t$ to the query history, update the revealed event set and
candidate mask, and refit the LGCP for the next iteration.
\\

\multicolumn{2}{l}{\textbf{end for}}
\\
\hline

\textbf{Return}
&
Spatial query centres, revealed event sets, per-iteration posterior intensity
maps, fitting times, and region-based optimization metrics.
\\

\hline
\end{tabular}

\caption{\small Detailed two-dimensional spatial implementation of ISBO.}
\label{alg:bo-inla-lgcp-2d}
\end{table}

\end{document}